\documentclass[11pt]{article}
\usepackage[utf8]{inputenc}
\usepackage[T1]{fontenc}
\usepackage{amsmath,amssymb,amsthm}
\usepackage{algorithm}
\usepackage{algpseudocode}
\usepackage{booktabs}
\usepackage{tabularx}
\usepackage{array}
\usepackage{graphicx}
\usepackage{caption}
\usepackage{hyperref}
\usepackage{geometry}
\usepackage{comment}
\usepackage{color}
\usepackage{subcaption}
\usepackage{authblk}

\newtheorem{theorem}{Theorem}

\newtheorem{proposition}{Proposition}
\newtheorem{corollary}{Corollary}

\title{Optimization Techniques for Solving the Minimum Red--Blue Separating Set Problem with an Application to Human--AI Text Separation}

\author[1,2]{Marko Djukanovi\'c\thanks{Corresponding author:\texttt{marko.dukanovic@ung.si}}}
\author[3]{Sta\v{s}a Vuji\v{c}i\'c Stankovi\'c \thanks{\texttt{stasa.vujicic.stankovic@matf.bg.ac.rs}}}
\author[2]{\v{Z}iga Zebec\thanks{\texttt{ziga.zebec@izum.si}}}
\author[4]{Christian Blum\thanks{\texttt{christian.blum@csic.es}}}

\affil[1]{University of Nova Gorica, Nova Gorica, Slovenia}
\affil[2]{Institute of Information Sciences (IZUM), Maribor, Slovenia}
\affil[3]{Faculty of Mathematics, University of Belgrade, Belgrade, Serbia}
\affil[4]{Artificial Intelligence Research Institute (IIIA-CSIC), Barcelona, Spain}

\date{}

\newcommand{\qns}{\mathit{qns}}

\begin{document}
	\maketitle

	\begin{abstract}
		We study the so-called Minimum Red–Blue Separating Set Problem (MRBSSP) which arises from computational geometry. In particular, we propose a compact Set-Cover ILP formulation, safe preprocessing rules, a Greedy Randomized Adaptive Search Procedure (GRASP), and two exact approaches: one is based on Logic-Based Benders Decomposition (LBBD) approach, and the second represents a novel Branch-and-Benders-Cut (BBC) approach. These two exact approaches are compared to ILP--alongside GRASP and the Variable Neighborhood Search (VNS) from literature. Experiments on 250 random graphs with up to 500 vertices reveal complementary strengths: ILP yields the best feasible solutions, LBBD provides the strongest lower bounds and smallest mean optimality gap, and BBC certifies optimality most frequently. We further introduce an NLP application distinguishing human- and AI-generated texts through vocabulary-based graph features. On 42 instances, BBC achieves the smallest mean optimality gap, certifying 16 optima, while GRASP outperforms VNS. Structural analysis of optimal vocabulary separators identifies 113 mandatory and 151 replaceable features across 272 tests, highlighting substantial variability among optimal solutions. Overall, our results demonstrate the effectiveness of complementary optimization strategies and the interpretability of optimal vocabulary-based discrimination.
	\end{abstract}

	\section{Introduction}
	\label{sec:introduction}
	
	Separating objects of different types by a small number of selected
	features, tests, or sensors is a recurring theme in combinatorial
	optimization, data analysis, and network design~\cite{balas1976set,acid2013algorithm,alber2004geometric,krak2019data,liu2019separate}. In graph settings, this
	principle leads naturally to separation problems in which the goal is not
	merely to dominate or cover vertices, but to distinguish pairs of vertices
	belonging to different classes. The Minimum Red-Blue Separating Set Problem
	(MRBSSP) captures this requirement in a particularly compact form. Given a
	graph $G=(V,E)$ and two disjoint sets of terminals, red vertices $R$ and
	blue vertices $B$, the task is to select a minimum-cardinality subset
	$S\subseteq V$ such that every red-blue pair $(r,b)\in R\times B$ is
	separated by at least one selected vertex. A vertex $v\in V$ separates the
	pair $(r,b)$ if its closed neighborhood $N[v]$ contains exactly one of
	$r$ and $b$. Thus, each selected vertex acts as a local test: it
	distinguishes precisely those terminal pairs for which one endpoint is
	inside its closed neighborhood and the other is outside.
	
	MRBSSP can be viewed as a structured Set-Cover problem~\cite{caprara2000algorithms,hochbaum1997approximating,christofides1975note}. The universe of
	elements to be covered is the Cartesian product $R\times B$, and each
	candidate vertex $v\in V$ covers the subset of red-blue pairs that it
	separates. The natural integer programming formulation therefore contains
	one binary variable for each vertex and one covering constraint for each
	red-blue terminal pair. This formulation is conceptually simple and results in
	a direct exact model, but it can become large even for moderate instances,
	since the number of constraints grows with $|R||B|$. Moreover, many of these
	constraints may be irrelevant for a large part of the search, because a
	candidate solution typically violates only a small subset of all red-blue
	pairs. This observation is the starting point for the decomposition and
	cut-generation methods~\cite{ralphs2006decomposition,saharidis2010improving} studied in this paper.
	
	The problem is closely related to several established lines of research.
	At the most general level, MRBSSP inherits the modeling structure and
	computational difficulty of Set Cover and Hitting Set~\cite{vinterbo2000minimal,Gainer2017minimal}, classical
	NP-hard problems for which integer programming, greedy heuristics, and
	approximation algorithms form the standard methodological toolkit
	\cite{Feige1998}. It is also related to Red-Blue Set Cover, introduced by
	Carr et al.~\cite{CarrEtAl2000}, where the objective is to cover all blue
	elements while minimizing the number of red elements covered. Although
	MRBSSP has a different graph-based separation semantics, both problems
	share the same red-blue modeling philosophy: one must satisfy conditions
	involving one color class while controlling how the other color class is
	affected. More directly, MRBSSP is connected to the Red-Blue Separation
	problem on graphs studied by Dev et al.~\cite{DevEtAl2023}, where a small
	set of vertices is selected so that every red-blue pair is distinguished
	by the closed neighborhood of at least one selected vertex. That work
	establishes the graph-theoretic and complexity foundations of the problem,
	including hardness results and approximation guarantees for the general
	case. Complementing this theoretical line, Milojkovi\'c~\cite{Milojkovic2025}
	recently proposed a Variable Neighborhood Search (VNS)~\cite{MladenovicHansen1997} heuristic for the
	MRBSSP, showcasing that neighborhood-based
	metaheuristics can provide competitive feasible solutions for graph
	instances. The present paper builds on this emerging computational
	perspective by focusing on exact and hybrid methods, using the VNS method from the literature~\cite{Milojkovic2025} and the newly introduced GRASP as
	strong heuristic reference baselines while emphasizing certification and
	dual-bound progress.
	
	From an algorithmic perspective, MRBSSP is particularly well suited to
	decomposition. Once a candidate set $S$ of selected vertices is fixed,
	assessing feasibility reduces to a simple set-intersection test: for every
	red-blue pair $(r,b)$, one verifies whether $S$ contains at least one
	vertex that separates the pair. If such a pair is not separated, the
	corresponding Set-Cover constraint is violated and can be added as a cut.
	This structure is a natural match for Logic-Based Benders Decomposition
	(LBBD)~\cite{HookerOttosson2003}, where a master problem proposes a candidate solution and a logical
	subproblem either verifies feasibility or returns violated constraints. It is also a natural match for
	Branch-and-Benders-Cut (BBC), where the same feasibility oracle is embedded
	inside a single branch-and-bound tree through lazy constraint generation~\cite{Pearce2019}, in the spirit of the combinatorial Benders cuts of Codato and Fischetti~\cite{CodatoFischetti2006}.
	The essential distinction is that iterative LBBD repeatedly solves a
	sequence of master problems, whereas BBC preserves the global search tree
	and injects cuts only when they are needed.
	
	The motivation for this study is therefore twofold. First, the direct ILP
	model for MRBSSP is easy to state but may be inefficient because it builds
	the full $|R||B|$ constraint family in advance. Second, purely heuristic
	methods such as GRASP or VNS can rapidly produce
	good feasible solutions, but they do not provide the dual bounds nor the optimality
	certification. This creates a natural
	methodological question: can the structure of MRBSSP be exploited to obtain
	exact methods that retain the certification strength of integer programming
	while avoiding unnecessary constraint generation and benefiting from
	high-quality heuristic incumbents?
	
	We investigate this question on two deliberately different benchmark
	families. The first is a broad \textsc{Random} benchmark composed of synthetically
	generated graphs with varying sizes and densities. Its role is to stress the
	algorithms under controlled structural variation and to permit a systematic
	comparison of primal performance, dual-bound progress, optimality gaps,
	certification rates, and heuristic stability. The second is an NLP-derived
	benchmark in which human- and AI-generated texts are represented through
	vocabulary-incidence graphs. In this family, the selected separator has a
	direct feature-selection interpretation: a minimum set of vocabulary
	features is sought that distinguishes every human--AI text pair. The NLP
	instances therefore complement the synthetic study with a real-data setting
	in which both computational performance and the structure of the resulting
	optimal separators can be interpreted.
	The two benchmark families serve different but complementary purposes.
	
	%\textsc{Random} instances provide scale, controlled difficulty, and broad algorithmic
	%stress testing, whereas the NLP instances evaluate whether the same
	%methodology remains effective on structured real-data graphs and whether the
	%selected optimal vertices admit a meaningful application-level
	%interpretation. The methods studied in this paper are therefore assessed not
	%	only by objective values and running-time behavior, but also by lower-bound
	%strength, optimality certification, repeated-run heuristic stability, and,
	%for the NLP benchmark, the rigidity and replaceability of optimal vocabulary
	%features.\\
	
	To the best of our knowledge, this work is among the earliest systematic exact
	computational studies of MRBSSP. It establishes optimality certificates for
	benchmark instances previously treated primarily by heuristic methods and
	compares three complementary exact strategies against two metaheuristic
	baselines across both synthetic and real-data benchmark families. The main
	contributions are as follows.
	\begin{itemize}
		\item We present a unified exact-optimization framework for MRBSSP based on
		the Set-Cover interpretation of red--blue pair separation.
		\item We formulate and compare three exact strategies, which, to the best of our knowledge, are the first ones developed for MRBSSP: a GRASP-warm-started ILP, an iterative LBBD method, and a BBC implementation based on lazy separation cuts within a single branch-and-bound tree.
		\item We develop a GRASP construction and local-improvement procedure used
		both as a standalone heuristic and as a source of high-quality incumbents for
		the exact methods, and compare it with a VNS baseline from the literature.
		\item We show how MRBSSP feasibility can be checked by a simple logical
		oracle that detects uncovered red--blue pairs and generates valid Set-Cover
		cuts, with only a limited number of heuristically selected cuts added per LBBD iteration or BBC
		callback invocation.
		\item We conduct a broad computational study on a \textsc{Random} benchmark of 250
		graphs spanning multiple sizes and densities. This benchmark is used for
		systematic algorithmic stress testing, including primal quality,
		lower-bound progress, optimality gaps, certification behavior, parameter
		tuning, and repeated-run stability of GRASP and VNS.
		\item We introduce a second, structurally different benchmark family derived
		from NLP vocabulary-separation instances for distinguishing human- and AI-generated texts, providing both a
		real-data computational benchmark and an interpretable feature-selection
		application of MRBSSP. For the NLP benchmark, we go beyond objective-value comparison and
		analyze the structure of certified optimal separators by forbidding selected
		vertices one at a time. This distinguishes mandatory vocabulary features
		that occur in every optimum from replaceable features that belong only to some optimal solutions.
	\end{itemize}
	
	\subsection{Notation and Problem Restatement}
	\label{sec:notation}
	
	Let $G = (V, E)$ be an undirected graph with disjoint vertex subsets
	$R \subseteq V$ (red terminals) and $B \subseteq V$ (blue terminals),
	$R \cap B = \emptyset$. For $v \in V$, let $N[v] = N(v) \cup \{v\}$ denote
	the closed neighborhood of $v$. A vertex $v$ \emph{separates} a red-blue
	pair $(r,b) \in R \times B$ if
	\[
	\bigl| N[v] \cap \{r, b\} \bigr| = 1,
	\]
	i.e., $v$'s closed neighborhood contains exactly one of the two terminals.
	For each pair $(r,b)$ define the \emph{coverage set}
	\[
	D(r,b) \;=\; \bigl\{\, v \in V \;:\; |N[v] \cap \{r,b\}| = 1 \,\bigr\}.
	\]
	A subset $S \subseteq V$ is a \emph{feasible red-blue separator} if it
	intersects $D(r,b)$ for every pair $(r,b) \in R \times B$, i.e.
	$S \cap D(r,b) \neq \emptyset \ \forall (r,b)$. The MRBSSP asks for a
	feasible separator of minimum cardinality:
	\[
	\text{MRBSSP:} \qquad \min_{S \subseteq V} \; |S|
	\quad \text{s.t.} \quad S \cap D(r,b) \neq \emptyset \ \ \forall (r,b) \in R \times B.
	\]
	This is structurally a Set-Cover / Hitting Set problem in which the
	``elements to be covered'' are the $|R|\times|B|$ red-blue pairs and the
	``sets'' are the per-vertex incidences induced by $D(\cdot,\cdot)$. \\

	The remainder of the paper is organized as follows.
	Section~\ref{sec:methods} introduces the classical Set-Cover ILP formulation
	and the preprocessing rules used by the exact approaches.
	Section~\ref{sec:grasp} presents the GRASP heuristic and its role in providing
	validated warm starts and upper bounds. Sections~\ref{sec:lbbd} and~\ref{sec:bbc}
	present the LBBD and BBC algorithms, respectively.
	Section~\ref{sec:experiments} reports the first benchmark study on the \textsc{Random}
	graph family, including parameter tuning and comparisons of primal quality,
	lower-bound strength, optimality gaps, certification behavior, and
	repeated-run heuristic stability.
	Section~\ref{sec:nlp-benchmark} presents the second benchmark family, derived
	from NLP vocabulary-separation instances for distinguishing human- and
	AI-generated texts. In addition to comparing all five approaches, it analyzes
	the structure of certified optimal vocabulary separators in terms of
	mandatory and replaceable features.
	Section~\ref{sec:conclusions} concludes the paper by synthesizing the evidence from these two benchmark
	families and outlining algorithmic and application-level directions for
	future work.

	\subsection{Classical ILP (Set-Cover) Formulation}
	\label{sec:ilp}
	
	Let define a binary decision variable $x_v \in \{0,1\}$ for every $v \in V$,
	with $x_v = 1$ if and only if $v \in S$, where $S$ denotes a solution of the tackled problem. The direct ILP formulation $\mathcal{M}$ is given as follows:
	\begin{align}
		&\min \quad  \sum_{v \in V} x_v \label{eq:ilp-obj}\\
		&\text{s.t.} \quad \sum_{v \in D(r,b)} x_v \;\geq\; 1
		\qquad \forall (r,b) \in R \times B, \label{eq:ilp-cover}\\
		& x_v \in \{0,1\} \qquad \forall v \in V. \label{eq:ilp-bin}
	\end{align}
	Constraint family~\eqref{eq:ilp-cover} is the bottleneck of this
	formulation: it contains $|R|\times|B|$ rows, many of
	which are \emph{never tight} at an optimal (or even near-optimal) solution
	on sparse, large instances, yet all of them must be generated, stored, and
	processed by the presolver and the branch-and-cut engine before the first
	relaxation is even solved. This motivates decomposing the model so that
	covering constraints are generated \emph{only when violated}, which is
	exactly the role played by the Benders-style methods of
	Sections~\ref{sec:lbbd}--\ref{sec:bbc}.
	
	Pre-processing (kernelization) can shrink~\eqref{eq:ilp-cover} 
	before the model is passed to a solver. A pair with $|D(r,b)| = 1$ forces
	its unique covering vertex into every feasible solution (an
	\emph{essential} vertex); rows whose coverage set is a strict superset of
	another row's are dominated and can be dropped, representing the \textit{row} dominance filtering. This is implemented in the function \textsc{RowDominanceReduction}(), returning non-dominated set of constraints of the basic model $\mathcal{M}$.  A vertex whose set of covered pairs is a strict subset of another surviving
	vertex's coverage is called dominated and can be removed, representing the \textit{column} dominance filtering.  This is implemented in the function \textsc{ColumnDominanceReduction}(), returning non-dominated set of vertices that are kept from the basic model $\mathcal{M}$.
	These are the classical data reduction rules for set covering, whose
	effect on large real-world instances was demonstrated by
	Weihe~\cite{Weihe1998}.  The reductions are applied iteratively. Let $V^{(k)}$ denote the currently
	surviving vertex set and $C^{(k)}$ the set of surviving constraints comprised of the variables from the set $V^{(k)}$, to keep the reduced instance internally consistent. Formally, the column coverage is computed over surviving vertices, i.e., 
	$C(v)=\{(r,b): v\in D(r,b)\cap V^{(k)}\}$, 
	and after every column-removal step each remaining separation set is
	explicitly synchronized with the current vertex set,
	\begin{equation}
		D(r,b) \leftarrow D(r,b)\cap V^{(k)}.
		\label{eq:safe-kernel-sync}
	\end{equation}
	Thus a vertex removed in an earlier dominance round cannot remain as a
	stale element of a later $D(r,b)$ set. As a remark, in the implementation, the coverage
	loop therefore iterates over $D(r,b)\cap V^{(k)}$, not over the unreduced
	set $D(r,b)$.
	
	Before any heuristic warm start is constructed, a full feasibility sanity
	check is performed. If some red--blue pair has $D(r,b)=\emptyset$, the
	instance is declared \textit{infeasible} and GRASP, presented in Section~\ref{sec:grasp} or any other effective procedure is not called. For feasible
	instances, any approximate methods $\mathcal{H}$ may provide a MIP start and an objective cutoff. The pseudocode of solving model $\mathcal{M}$ is presented in Algorithm~\ref{alg:ilp}.  % but only
	%when the returned heuristic separator passes the full MRBSSP feasibility
	%check. Hence an infeasible heuristic %andidate can never introduce an
	%invalid upper-bound constraint that may exclude the true optimum.  % The same
	%safeguards are applied in ILP, LBBD, and BBC.
	
	\begin{algorithm}[h]
		\caption{Classical Set-Cover ILP for MRBSSP with preprocessing and warm start}
		\label{alg:ilp}
		\begin{algorithmic}[1]
			\Require Graph $G=(V, E)$, terminal sets $R, B$, Heuristic procedure $\mathcal{H}$ to generate approximate solutions
			\Ensure Optimal separator $S^{\ast}$ or infeasibility certificate
			\State Compute $D(r,b)$ for every $(r,b) \in R \times B$
			\If{$D(r,b)=\emptyset$ for some $(r,b)$}
			\State \Return \textsc{Infeasible} \Comment{sanity check before GRASP}
			\EndIf
			\State $Ess \gets \{v:\exists(r,b),\,D(r,b)=\{v\}\}$ \Comment{essential vertices}
			\State  $V^{(k)} \gets$ \textsc{ColumnDominanceReduction}($\mathcal{M}$)
			\State $C^{(k)} \gets$ \textsc{RowDominanceReduction}( $\mathcal{M}$)
			\State $D(r,b)\gets D(r,b)\cap V^{(k)}$
			\State $\mathcal{M}^{'} \gets \left(\sum_{v \in V^{(k)}} x_v, C^{(k)},  \{ x_v \in \{0, 1\} \colon  v \in V^{(k)} \} \right)$ 
			\State $S^{\text{H}}\gets$ \Call{$\mathcal{H}$}{$G,R,B$} \Comment{Run an approximate method (e.g., GRASP from Section~\ref{sec:grasp})}
			\State $UB\gets | S^{\text{H}} |$; enable ILP start from $S^{\text{H}}$
			\State Fix essential vertices in $\mathcal{M}^{'}: \text{include } x_v = 1, v \in Ess$, $\sum_v x_v\le UB$
			\State $S^{\ast}\gets$ \Call{Solve}{$\mathcal{M}^{'}$} \Comment{Pass the model to a solver}
			\State \Return $S^{\ast}$
		\end{algorithmic}
	\end{algorithm}
	
	\section{GRASP Approach for the MRBSSP}
	\label{sec:grasp}
	
	The both designed exact approaches from
	Section~\ref{sec:lbbd} and  Section~\ref{sec:bbc} are seeded with an initial incumbent produced by a
	Greedy Randomized Adaptive Search Procedure (GRASP)~\cite{FeoResende1995}. This section describes
	the GRASP approach, since its constructive and local-search
	phases directly determine the quality of the warm start and, in turn, the
	upper-bound cut~\eqref{eq:ub-cut} available to the exact methods.
	
	\subsection{Overview}
	
	GRASP alternates between (i) a
	\emph{randomized greedy construction} phase that builds a feasible
	separator from scratch, biased by a randomized restricted candidate list
	so that successive iterations explore different regions of the solution
	space, and (ii) a \emph{local search} phase that greedily shrinks the
	constructed separator while preserving feasibility. The best feasible
	separator found before the applicable stopping criterion is reached is
	returned as $S^{\text{GRASP}}$. The resulting
	separator is supplied
	as a MIP start or used to define the upper-bound cutoff (see Eq.~\eqref{eq:ub-cut} in Section~\ref{sec:bbc}).

	During the execution time, GRASP repeatedly performs randomized	greedy construction followed by local improvement and retains the best feasible separator encountered. Pseudocode of the GRASP procedure is given in Algorithm~\ref{alg:grasp}.  Consequently, the GRASP solution actually	supplied to \textsc{BBC} and \textsc{LBBD} as a warm start is
	\emph{not} the same solution as the standalone $S^{\text{GRASP}}$ reported
	in the results table: the former results from only 10 iterations, whereas
	the latter is the best solution obtained during a 3600-second search. This
	distinction is essential for interpreting Section~\ref{sec:exp-grasp}
	correctly and is revisited there.
	
	\begin{algorithm}[h]
		\caption{GRASP for MRBSSP}
		\label{alg:grasp}
		\begin{algorithmic}[1]
			\Require Graph $G=(V,E)$, terminal sets $R,B$, candidate-list size $l_{\text{size}}>0$ 
			\Ensure Heuristic separator $S^{\text{GRASP}}$
			\State $S^{\text{GRASP}} \gets \emptyset$; \ $\mathit{cost}^{S^{\text{GRASP}}} \gets \infty$
			\While{stopping criterion has not been reached}
			\State $S^{'} \gets$ \Call{RandomizedGreedyConstruct}{$G,R,B,l_{\text{size}}$} \Comment{Section~\ref{sec:grasp-construct}}
			\State $S^{'} \gets$ \Call{LocalSearch}{$G,R,B,S^{'}$} \Comment{Section~\ref{sec:grasp-ls}}
			\If{$|S^{'}| <  \mathit{cost}^{S^{\text{GRASP}}}$}
			\State $S^{\text{GRASP}} \gets S'$; \ $\mathit{cost}^{S^{\text{GRASP}}} \gets |S'|$
			\EndIf
			\EndWhile
			\State \Return $S^{\text{GRASP}}$
		\end{algorithmic}
	\end{algorithm}
	
	Unlike the exact methods of Sections~\ref{sec:lbbd}--\ref{sec:bbc},
	GRASP carries no optimality certificate: it returns a feasible separator
	and its size, but no lower bound. Its role in this work is purely as a
	fast primal heuristic that (a) supplies a strong starting incumbent and
	(b) determines the right-hand side of the upper-bound
	cut~\eqref{eq:ub-cut} that tightens the BBC and LBBD master problems.
	
	\subsection{Randomized Greedy Construction}
	\label{sec:grasp-construct}
	
	The construction phase builds a separator one vertex at a time. At each
	step it maintains the set of red-blue pairs not yet separated by the
	partial solution and greedily selects a vertex that covers many of them,
	but---to allow GRASP's iterations to diversify rather than collapse onto
	a single deterministic greedy solution---the selection is randomized
	among the top-$l_{\text{size}}$ scoring candidates at each step (a restricted candidate
	list of size $l_{\text{size}}$) rather than always taking the single best-scoring
	vertex. %This is the constructive mechanism invoked as
	%\texttt{randomized\_greedy\_separator(G, R, B, k)} in the implementation;
	The parameter $l_{\text{size}}$ controls the trade-off between greedy quality (small
	$l_{\text{size}}$, close to deterministic best-first selection) and diversification
	(larger $l_{\text{size}}$, more randomness across GRASP iterations). The construction
	terminates once every pair in $R \times B$ is covered, which is guaranteed
	to occur in at most $|V|$ steps whenever the instance is separable.
	
	\subsection{Local Search: 1-Removal}
	\label{sec:grasp-ls}
	
	The local-search phase (LS)
	attempts to shrink the constructed separator $S'$ without ever leaving
	the feasible region. %It precomputes, once, the coverage map
	%$\texttt{pair\_cover\_map}(r,b) = D(r,b)$ for every pair (Section~\ref{sec:notation})
	%and the current set of covered pairs induced by $S_0$. 
	It repeatedly
	scans the vertices currently in the solution and tests, for each
	$v \in S$, whether removing $v$ leaves every pair still covered ---
	equivalently, whether $v$ is not the \emph{sole} remaining cover for any
	pair, i.e.,\ $S \cap D(r,b) \neq \{v\}$ for every $(r,b)$. The first vertex
	found to be removable is dropped, the covered-pairs bookkeeping is updated
	incrementally, and the scan restarts on the reduced solution (a
	first-improvement, restart-on-success strategy). The procedure terminates
	when a full scan finds no removable vertex, at which point $S$ is
	\emph{1-removal-minimal}: no single vertex can be deleted without
	uncovering some pair. The pseudocode of the procedure is shown in Algorithm~\ref{alg:ls}.
	
	\begin{algorithm}[h]
		\caption{Local Search (1-Removal) for MRBSSP}
		\label{alg:ls}
		\begin{algorithmic}[1]
			\Require Graph $G=(V,E)$, terminal sets $R,B$, feasible separator $S'$
			\Ensure 1-removal-minimal feasible separator $S$
			%\State Precompute $D(r,b)$ for every $(r,b) \in R \times B$ \Comment{\texttt{build\_pair\_cover\_map}}
			\State $S \gets S'$; \ $\mathit{Covered} \gets \{(r,b) : S \cap D(r,b) \neq \emptyset\}$
			\State $\mathit{improved} \gets \mathbf{true}$
			\While{$\mathit{improved}$}
			\State $\mathit{improved} \gets \mathbf{false}$
			\For{each $v \in S$} \Comment{Scan vertices in a fixed order}
			\State $\mathit{Covered}' \gets \mathit{Covered} \setminus \{(r,b) : S \cap D(r,b) = \{v\}\}$
			\Comment{pairs only $v$ covers}
			\If{$\mathit{Covered}' = R \times B$} \Comment{still fully covered without $v$}
			\State $S \gets S \setminus \{v\}$; \ $\mathit{Covered} \gets \mathit{Covered}'$
			\State $\mathit{improved} \gets \mathbf{true}$
			\State \textbf{break} \Comment{restart the scan on the reduced solution}
			\EndIf
			\EndFor
			\EndWhile
			\State \Return $S$
		\end{algorithmic}
	\end{algorithm}

	\section{Logic-Based Benders Decomposition (LBBD): Core Methodological Approach}
	\label{sec:lbbd}
	
	The central methodological idea of this work is to recognize that
	MRBSSP does not require all covering constraints~\eqref{eq:ilp-cover} to be
	present explicitly in order to define a correct master problem. A correct
	exact method requires only (a) a master problem that proposes candidate
	separators of minimum cardinality and (b) a procedure that can certify
	whether a candidate separator is feasible and, if not, explain
	\emph{why} it is infeasible in a form that can be fed back to the master
	problem. This is precisely the structure of \textit{Logic-Based Benders
		Decomposition} (LBBD): classical Benders decomposition~\cite{Benders1962}
	derives cuts from LP duality applied to a continuous subproblem, whereas
	LBBD derives cuts from \emph{logical inference} applied to a combinatorial
	feasibility subproblem. For MRBSSP, this subproblem reduces to a simple
	set-intersection test, which makes LBBD particularly well suited to this
	class of problems.
	
	\paragraph{Master Problem}
	
	The LBBD master problem $\mathcal{M}^{\text{LBBD}}$ retains only the decision variables and the
	objective from the ILP of Section~\ref{sec:ilp}. In its initial form, it
	contains the essential-vertex fixing constraints but \emph{no} red-blue
	covering constraints:
	\begin{align}
		\min \quad & \sum_{v \in V^{\text{nd}}} x_v \label{eq:master-obj}\\
		\text{s.t.} \quad & x_v = 1 \qquad \forall v \in Ess
		\quad \text{(essential vertices)}\label{eq:master-essential}\\
		& x_v \in \{0,1\} \qquad \forall v \in V^{\text{nd}}, \label{eq:master-bin}
	\end{align}

	where the set $V^{\text{nd}}$ represents the set of non-dominated candidate vertices obtained by  applying
	row- and column-dominance reductions on the  considered instance. The resulting reduced separation coverage and surviving candidate vertices are then passed to the LBBD procedure.
	
	As shown in Algorithm~\ref{alg:lbbd}, the implementation may pass a GRASP
	separator $S^{\text{GRASP}}$ to the MIP solver at each master solve. It is used
	as a MIP start and yields the valid cutoff
	$\sum_{v\in V^{\text{nd}}}x_v\le |S^{\text{GRASP}}|$ (equivalently, after fixing the
	essential vertices,
	$\sum_{v\in V^{\text{nd}}\setminus Ess}x_v\le |S^{\text{GRASP}}|-|Ess|$). 
	
	At each iteration $k$, solving the current relaxed master using this GRASP
	warm start yields a candidate $x^{(k)}$, which induces the candidate
	separator
	$S^{(k)} = \{v : x_v^{(k)} = 1, v \in V^{\text{nd}}\}$.
	Because only a limited number of covering constraints are present during
	the early iterations, $S^{(0)}$ is typically the empty set (or consists
	only of the essential vertices) unless the solver accepts the supplied
	incumbent during the search. The master objective value remains a valid
	\emph{lower bound} on the true optimum at every iteration, since the master
	contains only valid constraints and a valid incumbent cutoff.
	
	\paragraph{Subproblem (Logical Feasibility Oracle)}
	
	Given a candidate $S^{(k)}$, the LBBD subproblem is \emph{not} an LP, but
	rather a logical feasibility check:
	\[
	\text{Is } S^{(k)} \cap D(r,b) \neq \emptyset
	\quad \text{for every } (r,b) \in R\times B \,?
	\]
	This test decomposes completely over pairs and can be evaluated for all
	pairs in $O(|R||B|)$ time, given precomputed (kernelized) $D(r,b)$ sets.
	These sets can themselves be computed once in
	$O(|R||B||V|)$ preprocessing time and reused across all iterations and all 
	considered algorithms. For a maximally $q_{\text{LBBD}}$ pairs found to be \emph{uncovered} by $S^{(k)}$, the
	logic-based Benders feasibility cut
	\begin{equation}
		\sum_{v \in D(r,b)} x_v \;\geq\; 1
		\label{eq:benders-cut}
	\end{equation}
	is added into the master problem $\mathcal{M}^{\text{LBBD}}$. These inequalities simply
	restate the necessary condition that the pair $(r,b)$ must be separated
	by at least one selected vertex. Note that this feasibility cut is
	syntactically identical to the corresponding row~\eqref{eq:ilp-cover} of
	the classical ILP. The distinction is that LBBD generates such rows
	dynamically, only when they are violated by a master solution, rather than
	including the complete family of $|R||B|$ constraints a priori.
	
	\paragraph{Controlled Multi-Cut Strategy}
	
	At iteration $k$, the oracle first identifies the complete set
	$\mathit{Violated}$ of pairs not separated by the current candidate.
	Instead of adding every corresponding row at once, the implementation
	selects a bounded subset
	$\mathit{Selected} \subseteq \mathit{Violated} $ with
	$|\mathit{Selected}|\leq q_{\mathrm{LBBD}}$, where
	$q_{\mathrm{LBBD}}$ is a user-defined maximum number of cuts per iteration.
	The default ranking prioritizes pairs with the smallest separation sets
	$|D(r,b)|$, since these rows are typically more restrictive and cheaper to
	process. When $q_{\mathrm{LBBD}}=1$, the procedure reduces to a single-cut
	scheme, whereas a sufficiently large value recovers the all-cuts strategy.
	This controlled multi-cut policy prevents the first relaxed master
	solution from immediately generating the entire reduced ILP, while
	preserving exactness because at least one violated valid cut is appended
	whenever the current candidate is infeasible.
	
	\begin{algorithm}[h]
		\caption{Iterative Logic-Based Benders Decomposition for MRBSSP}
		\label{alg:lbbd}
		\begin{algorithmic}[1]
			\Require Graph $G=(V,E_G)$, terminal sets $R,B$, cut limit
			$q_{\mathrm{LBBD}}\ge1$, $l_{\text{size}}\ge 1$
			\Ensure Optimal separator $S^{\ast}$ or infeasibility certificate
			\State Compute $D(r,b)$ for every $(r,b)\in R\times B$
			\If{$D(r,b)=\emptyset$ for some pair}
			\State \Return \textsc{Infeasible}
			\EndIf
			\State Compute essential vertices; kernelize and compute non-dominated
			constraints $C^{\text{nd}}$ and vertices $V^{\text{nd}}$
			\State Synchronize $D(r,b)\gets D(r,b)\cap V^{\text{nd}}$
			\State $S^{\text{GRASP}}\gets$
			\Call{GRASP}{$G,R,B,l_{\text{size}}$}
			\State Build the kernelized master $\mathcal{M}$ and add its MIP start
			and cutoff
			\State $k\gets0$
			\Repeat
			\State $x^{(k)}\gets$ \Call{Solve}{$\mathcal{M}^{\text{LBBD}}$, $S^{\text{GRASP}}$}
			\Comment{warm start/cutoff, if available}
			\State $S^{(k)} \gets \{ v \in V : x_v^{(k)} = 1 \}$
			\State $\mathit{Violated} \gets
			\bigl\{\, (r,b) \in R\times B :
			S^{(k)} \cap D(r,b) = \emptyset \,\bigr\}$
			\State $\mathit{Selected} \gets$
			\Call{SelectCuts}{$\mathit{Violated},q_{\mathrm{LBBD}}$}
			\Comment{at most $q_{\mathrm{LBBD}}$ pairs; smallest $|D(r,b)|$ first}
			\For{each $(r,b) \in \mathit{Selected}$}
			\State Add cut $\sum_{v \in D(r,b)} x_v \geq 1$ to  $\mathcal{M}^{\text{LBBD}}$
			\EndFor
			\State $k \gets k+1$
			\Until{$\mathit{Violated} = \emptyset$}
			\State $S^{\ast} \gets S^{(k-1)}$
			\State \Return $S^{\ast}$
		\end{algorithmic}
	\end{algorithm}
	
	The complete LBBD procedure is shown in Algorithm~\ref{alg:lbbd}.
	
	\subsection{Correctness and Finite Termination}
	
	\begin{proposition}
		\label{prop:lbbd-correct}
		Algorithm~\ref{alg:lbbd} terminates after a finite number of iterations
		and returns an optimal solution to MRBSSP.
	\end{proposition}
	
	\begin{proof}
		\textbf{(Finite termination.)}
		There are exactly $|R||B|$ distinct pairs and, hence, at most
		$|R||B|$ distinct cuts of the form~\eqref{eq:benders-cut} can ever be
		generated. Each pair induces at most one distinct cut. Even if a pair is
		violated again by a later candidate solution, its corresponding cut is
		generated at most once in an implementation that explicitly tracks
		previously added cuts. Since $V$ is finite, after $k$ iterations the
		feasible region of the master problem is a subset of
		$\{0,1\}^{|V|}$ intersected with at most $k$ additional linear
		inequalities. Once all $|R||B|$ possible cuts have been added in the
		worst case, the master becomes identical to the full ILP of
		Section~\ref{sec:ilp}, and any optimal solution to it is directly
		feasible. Therefore, the loop must terminate within at most
		$|R||B|+1$ iterations.
		
		\textbf{(Optimality.)}
		At every iteration $k$, $\mathcal{M}^{\text{LBBD}}$ is a relaxation of the full
		ILP~\eqref{eq:ilp-obj}--\eqref{eq:ilp-bin}: it contains a strict subset
		of the covering constraints, together with the valid essential-vertex
		fixing~\eqref{eq:master-essential}. Hence,
		$z(\mathcal{M}^{\text{LBBD}}) \leq z^{\ast}$ for every $k$, where $z^{\ast}$
		is the true MRBSSP optimum and $z(\mathcal{M}^{\text{LBBD}})$ is the optimal
		value of the master problem at iteration $k$, and therefore a valid
		lower bound. The loop terminates only when $S^{(k)}$ satisfies
		\emph{every} covering constraint~\eqref{eq:ilp-cover}, i.e.,
		$\mathit{Violated} = \emptyset$. This means that $S^{(k)}$ is feasible
		for the full ILP and, therefore,
		$|S^{(k)}| \geq z^{\ast}$. Combined with
		$z(\mathcal{M}^{\text{LBBD}}) = |S^{(k)}| \leq z^{\ast}$, since $S^{(k)}$ is
		the master's own optimal solution at that iteration, we obtain
		$|S^{(k)}| = z^{\ast}$, i.e., $S^{(k)}$ is optimal.
	\end{proof}
	
	This argument is the logic-based analogue of the classical Benders
	feasibility-cut argument: instead of a dual certificate proving that a
	continuous subproblem is infeasible or suboptimal, the certificate here is
	the combinatorial fact
	$S^{(k)} \cap D(r,b) = \emptyset$, which logically implies the cut.
	
	\paragraph{Limitation of the Iterative Scheme}
	
	Algorithm~\ref{alg:lbbd} re-solves the master MIP \emph{from scratch} at
	every iteration (the call
	\textsc{Solve}$(\mathcal{M},S^{\text{GRASP}})$. Although the GRASP warm
	start and cutoff are re-applied at each solve, the branch-and-bound tree
	explored in the previous iteration is discarded after each new batch of
	cuts is appended. On instances requiring many iterations, this repeated
	restart overhead can dominate the total running time, motivating the
	tighter integration described in the next section.

	\section{Branch-and-Benders-Cut (BBC)}
	\label{sec:bbc}
	
	Branch-and-Benders-Cut eliminates the restart overhead of
	Algorithm~\ref{alg:lbbd} by embedding the same logical feasibility oracle
	directly within a \emph{single} branch-and-bound search tree through a
	\textit{lazy constraint callback}. Rather than solving the master to full
	MIP optimality, identifying violations, and re-solving it, BBC intercepts
	the search when an integer-feasible candidate is found, checks it
	immediately, and either accepts it or injects the missing cuts into the
	\emph{same} tree before the search continues. The pseudocode of BBC is
	shown in Algorithm~\ref{alg:bbc-master}.
	
	\paragraph{Master/Callback Architecture}
	
	The BBC master problem is identical in structure to the LBBD master
	(Equations~\eqref{eq:master-obj}--\eqref{eq:master-bin}). The algorithm
	first checks whether any red--blue pair has an empty separation set; if so,
	it terminates with an infeasibility certificate. Otherwise, it runs GRASP,
	and the resulting solution is used as a MIP start and to define the valid
	upper-bound cutoff
	\begin{equation}
		\sum_{v \in V^{\text{nd}} \setminus Ess} x_v \;\leq\;
		|S^{\text{GRASP}}| - |Ess|.
		\label{eq:ub-cut}
	\end{equation}
	
	The master is solved \emph{once}, but the branch-and-bound engine invokes
	the $\textsc{FeasibilityOracle}$ callback
	(see Algorithm~\ref{alg:bbc-callback}) whenever it discovers a new
	integer-feasible candidate solution, \emph{before} that candidate is
	accepted as the incumbent. At each callback invocation, the oracle adds at
	most $q_{\mathrm{BBC}}$ lazy constraints rather than all currently
	violated rows. The implemented ranking prioritizes pairs with the smallest
	$|D(r,b)|$. Since even a single violated lazy constraint excludes the
	current candidate, this policy remains exact while avoiding the immediate
	reconstruction of the complete reduced ILP at the first callback.
	
	\begin{algorithm}[h]
		\caption{Branch-and-Benders-Cut for MRBSSP --- Master Driver}
		\label{alg:bbc-master}
		\begin{algorithmic}[1]
			\Require Graph $G=(V, E)$, terminal sets $R,B$, cut limit
			$q_{\mathrm{BBC}}\ge1$, $l_{\text{size}}\ge 1$
			\Ensure Optimal separator $S^{\ast}$ or infeasibility certificate
			\State Compute $D(r,b)$ for every $(r,b)\in R\times B$
			\If{$D(r,b)=\emptyset$ for some pair}
			\State \Return \textsc{Infeasible}
			\EndIf
			\State $S^{\text{GRASP}}\gets$
			\Call{GRASP}{$G,R,B,l_{\text{size}}$}
			\State Compute essential vertices; kernelize and synchronize
			$D(r,b)\gets D(r,b)\cap V^{\text{nd}}$
			
			\State Build master $\mathcal{M}$ with
			objective~\eqref{eq:master-obj},  kernelization, and essential nodes
			\State Add cutoff~\eqref{eq:ub-cut} and register the MIP start from
			$S^{\text{GRASP}}$
			\State Register callback
			$\textsc{FeasibilityOracle}(q_{\mathrm{BBC}})$ to be invoked for every
			integer-feasible candidate
			\State $S^{\ast} \gets$
			\Call{BranchAndBound}{$\mathcal{M}$,$S^{\text{GRASP}}$}
			\Comment{single search tree; never restarted}
			\State \Return $S^{\ast}$
		\end{algorithmic}
	\end{algorithm}
	
	Whenever the solver encounters an integer-feasible master candidate, the
	callback reconstructs the complete candidate separator, including all
	essential vertices fixed to 1 during preprocessing. Thus,
	\[
	S^{(k)} =
	Ess \cup \{v\in V^{\text{nd}}\setminus Ess : x_v^{(k)}=1\}.
	\]
	The feasibility oracle then checks this complete separator against all
	surviving red--blue separation requirements.
	
	\begin{algorithm}[h]
		\caption{$\textsc{FeasibilityOracle}$ --- Lazy Constraint Callback}
		\label{alg:bbc-callback}
		\begin{algorithmic}[1]
			\Require Current integer-feasible candidate values $x^{(k)}$ at a
			search node; essential-vertex set $Ess$; cut limit
			$q_{\mathrm{BBC}}$
			\State $S^{(k)} \gets
			Ess \cup \{v\in V^{\text{nd}}\setminus Ess : x_v^{(k)}=1\}$
			\State $\mathit{Violated} \gets
			\bigl\{(r,b)\in R\times B :
			S^{(k)} \cap D(r,b)=\emptyset\bigr\}$
			\If{$\mathit{Violated} = \emptyset$}
			\State \textbf{accept} $x^{(k)}$ as a valid incumbent
			\Comment{no lazy constraint is required}
			\Else
			\State $\mathit{Selected} \gets$
			\Call{SelectCuts}{$\mathit{Violated},q_{\mathrm{BBC}}$}
			\Comment{at most $q_{\mathrm{BBC}}$ pairs; smallest $|D(r,b)|$ first}
			\For{each $(r,b) \in \mathit{Selected}$}
			\State Add lazy feasibility cut
			$\sum_{v \in D(r,b)} x_v \geq 1$ to the active search tree
			\EndFor
			\Comment{The current candidate violates the newly added cut(s) and is rejected by the solver.}
			\EndIf
		\end{algorithmic}
	\end{algorithm}
	
	\paragraph{Why BBC Preserves the Search Tree}
	
	The crucial structural difference from Algorithm~\ref{alg:lbbd} is that
	$\textsc{FeasibilityOracle}$ does not stop or restart the branch-and-bound
	process; instead, it is invoked \emph{synchronously} within the solver's
	incumbent-acceptance routine. Adding a cut violated by the current
	candidate is sufficient for the solver to treat that candidate as
	infeasible. The search then continues within the same branch-and-bound
	tree, with the newly added valid inequalities remaining active throughout
	the remainder of the search and being inherited by all subsequently
	explored nodes, including siblings and descendants of already processed
	nodes that may benefit from the new cut through bound tightening. This is
	precisely what eliminates the restart overhead identified as the
	limitation of Section~\ref{sec:lbbd}: LBBD restarts the branch-and-bound
	search after each new batch of cuts, whereas BBC maintains a single search
	tree throughout.
	
	\subsection{Correctness}
	
	\begin{theorem}
		\label{thm:bbc-correct}
		Algorithm~\ref{alg:bbc-master}
		(with callback~\ref{alg:bbc-callback})
		terminates and returns an optimal solution to MRBSSP.
	\end{theorem}
	
	\begin{proof}
		Branch-and-bound over the finite domain
		$\{0,1\}^{|V \setminus E|}$ terminates in finite time regardless of how
		many lazy cuts are injected during the search, since each cut only
		\emph{removes} candidate solutions and the underlying search tree
		remains finite. This is a standard property of lazy-constraint
		branch-and-cut: although at most $q_{\mathrm{BBC}}$ cuts are added per
		callback invocation, every infeasible integer candidate triggers at
		least one valid cut. The lazy-constraint pool therefore grows
		monotonically and is bounded above by $|R||B|$ distinct rows, exactly
		as argued in Proposition~\ref{prop:lbbd-correct}.
		
		For optimality, observe that the solver reports an incumbent
		$x^{\ast}$ as final only if (a) it is provably optimal with respect to
		every constraint currently present in the model and (b)
		$\textsc{FeasibilityOracle}$ finds no violated pair for it. Otherwise,
		a new cut would be added and $x^{\ast}$ would be excluded by that cut,
		contradicting (a). Condition (b) means that
		$S^{\ast} = \{v : x^{\ast}_v = 1\}$ satisfies
		$S^{\ast} \cap D(r,b) \neq \emptyset$ for \emph{every}
		$(r,b) \in R \times B$. Thus, $S^{\ast}$ is feasible for the full
		ILP~\eqref{eq:ilp-obj}--\eqref{eq:ilp-bin}, and therefore
		$|S^{\ast}| \geq z^{\ast}$. Conversely, every cut that is added is a
		logically valid inequality
		(Equation~\eqref{eq:benders-cut} restates a necessary condition for
		feasibility), so the relaxed model solved at termination remains a
		valid relaxation of the full ILP, yielding
		$|S^{\ast}| \leq z^{\ast}$ by the same relaxation argument used in
		Proposition~\ref{prop:lbbd-correct}. Hence,
		$|S^{\ast}| = z^{\ast}$.
	\end{proof}
	
	\begin{corollary}
		LBBD (Algorithm~\ref{alg:lbbd}) and BBC
		(Algorithm~\ref{alg:bbc-master}) are equivalent in terms of the cut
		family they may generate: both are restricted to the same universe of
		at most $|R||B|$ valid inequalities~\eqref{eq:benders-cut} and
		therefore solve to the same optimal objective value $z^{\ast}$ on any
		instance. They differ only in \emph{when} and \emph{how} those cuts are
		discovered and exploited during the search, which affects running time
		and the number of explicitly materialized constraints, but not solution
		quality or correctness.
	\end{corollary}
	
	\subsection{Discussion: Why BBC Is a Natural Fit for MRBSSP}
	
	Three structural properties of MRBSSP jointly make it a favorable
	candidate for the Logic-Based Branch-and-Benders-Cut paradigm, in contrast
	to problems whose Benders subproblems are themselves expensive LPs:
	\begin{enumerate}
		\item \textbf{A compact master.}
		The master retains only $|V|$ binary variables and the simple
		essential-vertex fixing
		constraints~\eqref{eq:master-essential}; it contains none of the
		$|R||B|$ structural rows.
		
		\item \textbf{A cheap feasibility oracle.}
		Checking
		$S^{(k)} \cap D(r,b) \neq \emptyset$ for all pairs requires only a
		simple set-intersection sweep,
		$O(|R||B|)$ given precomputed coverage sets, with no LP or
		combinatorial subproblem to solve.
		
		\item \textbf{Simple, sparse, valid cuts.}
		Each generated cut is a single $\geq 1$ row over $D(r,b)$,
		syntactically identical to a row of the original ILP and requiring no
		auxiliary variables or big-$M$ constants.
	\end{enumerate}
	
	Combining this structure with safe iterative kernelization
	(essential vertices and row/column dominance with synchronization of the
	surviving separation sets) and a feasibility-validated GRASP warm start
	substantially reduces the number of explicitly materialized constraints
	while still exploiting the solver's native branch-and-cut engine. When
	GRASP is feasible, it supplies both an initial incumbent and the
	upper-bound cut~\eqref{eq:ub-cut}; otherwise, both are omitted without
	affecting exactness.
	
	\section{Experimental Results}
	\label{sec:experiments}
	
	The \textsc{Random} benchmark contains 250 instances arranged in 25 $(n,p)$ groups, where $n$ and $p$ represent the graph size and the graph density, respectively, 
	with 10 instances generated for each combination of $n\in\{20,50,100,200,500\}$ and
	$p\in\{0.01,0.05,0.1,0.2,0.5\}$. 
	Note that five sparse groups,
	$(20, 0.05)$, $(20, 0.1)$, $(50,0.01)$, $(100,0.01)$, and $(200,0.01)$, are identified
	as infeasible by the common feasibility checks involved in each of the exact approaches. Therefore, the results on those instances are not reported here. Thus, the quantitative comparison is reported for 200 feasible instances, i.e., 20 groups.
	
	Five approaches are evaluated: ($i$) the compact ILP with the $\mathcal{H}$=GRASP warm start  (\textsc{Ilp}); ($ii$) Branch-and-Benders-Cut with the GRASP warm start (\textsc{Bbc}); ($iii$) 
	Logic-based Benders decomposition with GRASP warm start (\textsc{Lbbd}); ($iv$)  Standalone (long-run) 
	\textsc{GRASP} approach (\textsc{Grasp}); ($v$) \textsc{VNS} baseline from the literature (\textsc{Vns}).
	
	All algorithms are implemented in Python~3.11, with the exact optimization
	models implemented using DOcplex and solved with IBM ILOG CPLEX~22.1.
	The experiments were conducted on the VEGA high-performance computing system
	hosted at IZUM, Maribor. The cluster comprises 960 compute nodes equipped with
	AMD EPYC~7H12 processors operating at 2.35~GHz. To ensure a consistent
	computational setting across all tested approaches, all experiments were
	executed in single-threaded mode.
	
	Each exact method was given a 3600-second  time limit after preprocessing and warm-start generation.
	On the \textsc{Random} instances these pre-solve phases take from several
	milliseconds to a few hundred seconds. On the NLP instances they take
	considerably longer: according to the LBBD logs, 123 to 158~s for
	$\qns=100$, 482 to 881~s for $\qns=150$, 1752 to 3979~s for $\qns=200$ and
	4274 to 6563~s for $\qns=250$. The time limits reported for the exact methods
	therefore refer to the solver phase only.
	The stochastic heuristic approaches were
	run independently ten times on each feasible instance, with every standalone (long-run)
	\textsc{Grasp} and \textsc{Vns} run subject to a 3600-second wall-clock limit. For \textsc{Grasp}, the
	solution reported for a run is the best feasible separator obtained before
	the time limit is reached. The substantially cheaper 10-iteration \textsc{Grasp} is used for the warm-start, 
	to initialize the exact methods as a separate procedure. The both \textsc{Grasp} settings utilized the size of restricted list $l_{\text{size}}$ to 5, selected after several preliminary results.  
	For each heuristic, the ten runs are
	first averaged within an instance and only then aggregated by $(n,p)$, so
	every graph receives equal statistical weight.

	\subsection{Algorithm tuning}
	\label{sec:algorithm-tuning}
	
	\textsc{Bbc} and \textsc{Lbbd} each expose one strategically important algorithm-specific
	parameter. We denote by $q_{\mathrm{BBC}}$ the maximum number of violated
	separation constraints added during one \textsc{Bbc} callback and by
	$q_{\mathrm{LBBD}}$ the maximum number added during one \textsc{Lbbd} master iteration.
	The parameter controls a natural trade-off. Small values keep each generated
	cut batch selective but may require many callbacks or repeated master solves,
	whereas very large values reduce the number of separation rounds at the price
	of substantially enlarging the current relaxation or master problem.
	
	We tune these parameters on dedicated subsets of both the \textsc{Random} and NLP
	benchmarks before fixing the values used in the final experiments. The tested
	values are
	\[
	q_{\mathrm{LBBD}}, q_{\mathrm{BBC}} \in\{1,2,5,10,50,100, 500\}.
	\]
	Parameter selection is
	based primarily on the final optimality gap and the number of proven optima;
	primal and dual bound quality and the number of separation rounds are used as
	secondary diagnostics.

	\subsubsection{Tuning on the \textsc{Random} benchmark}
	\label{sec:tuning-random}
	
	The completed \textsc{Random} tuning study evaluates the full parameter grid for both \textsc{Bbc} and \textsc{Lbbd}. The tuning set contains 20 instances, selecting one instance at random from each group of the \textsc{Random} benchmark. The aggregated results are reported in Appendix~\ref{sec:appendix-random-tuning}.
	
	The completed \textsc{Bbc} experiment confirms $q_{\mathrm{BBC}}=5$ as the
	\textsc{Random}-specific choice. It yields the smallest mean gap
	(25.39\%) and is the only
	tested value proving 10 of the 20 tuning instances optimal. The $q_{\textrm{BBC}}=500$ run gives mean gap
	26.16\% and 9 optimality
	certificates. Thus, increasing the callback batch from 5 to 500 reduces the
	average number of callbacks from
	241.45 to
	46.75, but increases the
	number of lazy rows materialized and does not improve final certificate quality.
	This confirms that \textsc{Bbc} benefits from a relatively selective cut batch inside
	one branch-and-bound tree.
	
	For \textsc{Lbbd}, $q_{\mathrm{LBBD}}=100$ remains the strongest \textsc{Random}-specific
	value, with mean gap
	24.85\% and 9 proven
	optima. At $q_{\mathrm{LBBD}}=500$, the number of master iterations falls further from
	10.35 to
	5.30, but the mean gap
	worsens to 25.86\% and the
	number of certificates falls to 8. Hence the complete \textsc{Random} tuning grid shows
	a clear over-batching regime beyond $q_{\mathrm{LBBD}}=100$ for \textsc{Lbbd}.
	
	\begin{figure}[htbp]
		\centering
		
		\begin{minipage}[t]{0.48\textwidth}
			\centering
			\includegraphics[width=\textwidth]{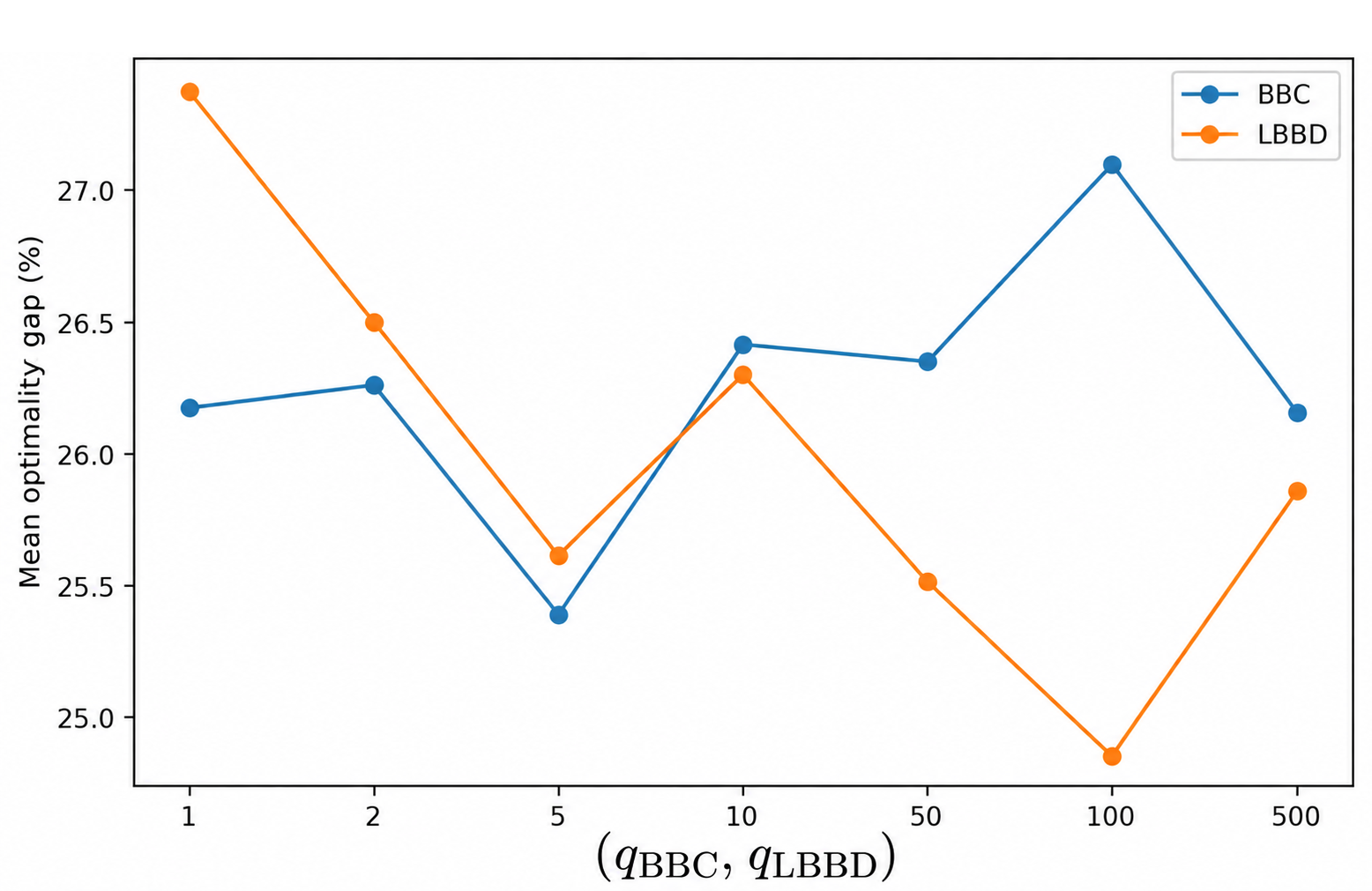}
			\captionof{figure}{Mean optimality gap on the complete \textsc{Random} tuning grid across different
				$q=q_{\text{LBBD}}, q_{\text{BBC}}$ values.}
			\label{fig:tuning-random-gap}
		\end{minipage}
		\hfill
		\begin{minipage}[t]{0.48\textwidth}
			\centering
			\includegraphics[width=\textwidth]{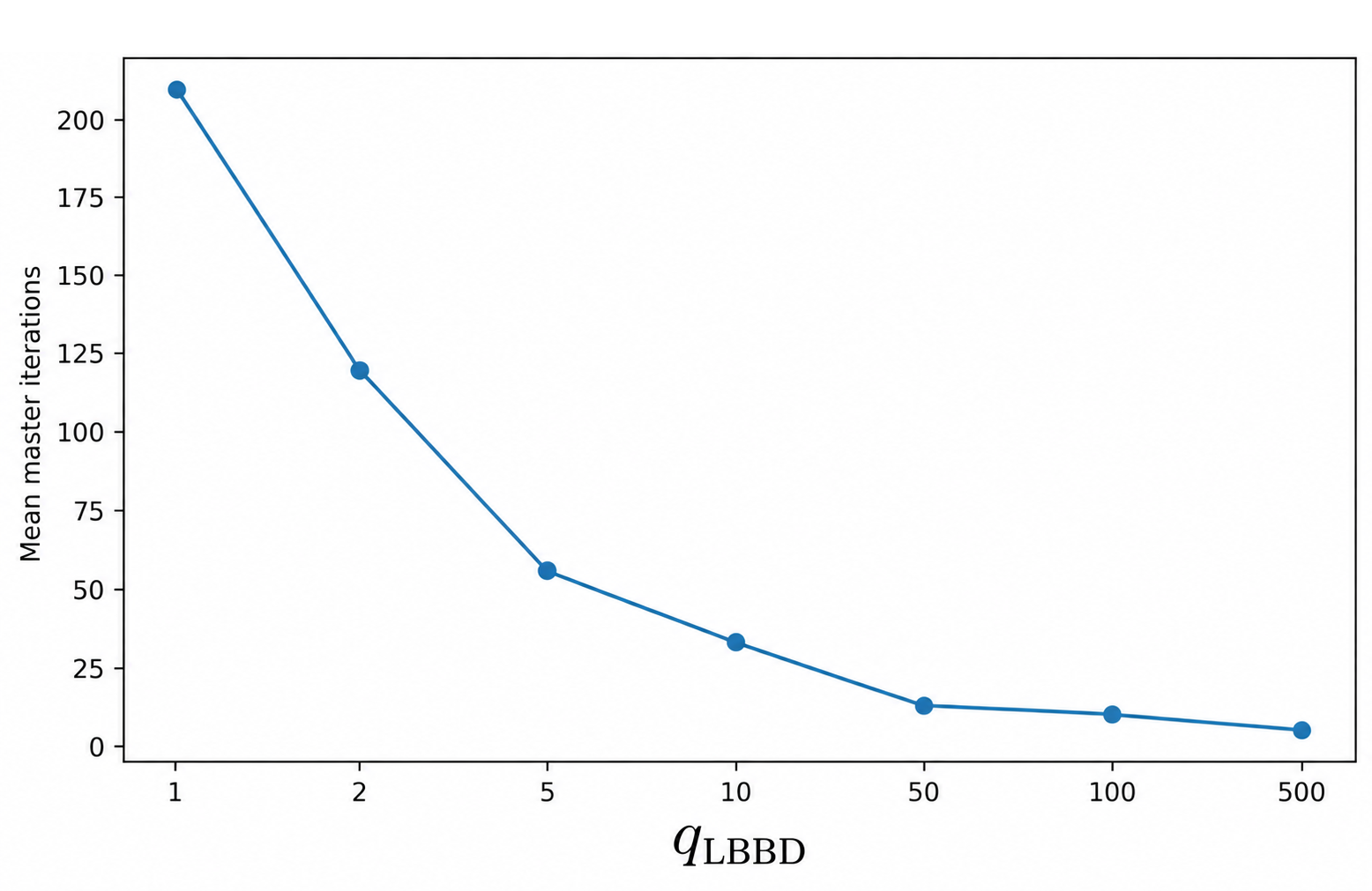}
			\captionof{figure}{Mean number of LBBD master iterations on the \textsc{Random} benchmark as the maximum
				number of cuts per iteration increases.}
			\label{fig:tuning-random-lbbd-it}
		\end{minipage}
		
	\end{figure}

	The completed benchmark-specific tuning therefore selects
	\[
	(q_{\mathrm{BBC}},q_{\mathrm{LBBD}})
	=(5,100)\text{ for \textsc{Random}}.
	\]
	
	%For the cross-benchmark comparison we use one common configuration per
	%algorithm. For BBC,
	%$q_{\mathrm{BBC}}=50$ still gives the smallest pooled mean gap
	%(17.88\%) and ties $q=5$ for the largest number of
	%optimality certificates (20/34). The aggressive
	%$q_{BBC}/q_{LLBD}=500$ setting yields a larger pooled mean gap (18.21\%) and only
	%19/34 certificates, so the new \textsc{Random} results
	%strengthen rather than change the choice of $q_{\mathrm{BBC}}=50$.
	
	%For LBBD, $q_{\mathrm{LBBD}}=100$ remains the best common value, with
	%the smallest pooled mean gap (16.85\%) and
	%19/34 proven optima. Although $q=500$ is the best
	%NLP-specific setting, its weaker \textsc{Random} behavior raises the pooled gap to
	%17.19\% and reduces the certificate count to
	%18/34. Consequently, the final common configuration
	%remains
	%\[
	%\boxed{q_{\mathrm{BBC}}=50,\qquad %q_{\mathrm{LBBD}}=100}.
	%\]
	
	%\begin{figure}[htbp]
	%	\centering
	%	\includegraphics[width=0.76\textwidth]{figures/tuning_combined_gap.png}
	%	\caption{Pooled mean optimality gap over the complete \textsc{Random} and NLP tuning
		%		grids. All seven tested values are now available for both algorithms.}
	%	\label{fig:tuning-combined-gap}
	%\end{figure}

	\subsection{Overall exact-method behavior}
	
	After tuning, \textsc{Bbc} is run with $q_{\mathrm{BBC}}=5$ and \textsc{Lbbd} with
	$q_{\mathrm{LBBD}}=100$ on the complete \textsc{Random} benchmark. \textsc{Vns} uses the settings from~\cite{Milojkovic2025}, and \textsc{Grasp} utilizes the same configuration reported in Section~\ref{sec:grasp}. All three exact methods use the same 10-iteration \textsc{Grasp} warm-start procedure; otherwise, \textsc{Ilp} uses CPLEX's default settings.
	Table~\ref{tab:random-summary}
	summarizes the resulting exact-method performance on all 200 feasible instances.
	
	\begin{table}[htbp]
		\centering
		\caption{Overall results of the exact methods on the 200 (feasible) \textsc{Random} instances after tuning.}
		\label{tab:random-summary}
		\small
		\begin{tabular}{lrrrr}
			\toprule
			Method & Mean UB & Mean LB & Mean gap & Optimal \\
			\midrule
			\textsc{Ilp} & 18.83 & 12.40 & 29.54\% & 90/200 \\
			\textsc{Bbc} & 19.09 & 13.11 & 25.97\% & \textbf{92/200} \\
			\textsc{Lbbd} & 20.25 & 13.63 & \textbf{24.37}\% & 91/200 \\
			\bottomrule
		\end{tabular}
	\end{table}
	
	\textsc{Ilp} remains the strongest exact method on the primal side, with mean UB
	18.83. Tuned \textsc{Lbbd} gives the strongest mean lower bound
	(13.63) and the smallest mean gap
	(24.37\%), while tuned \textsc{Bbc} proves the largest number
	of instances solved to optimality (92/200). Instance-wise, \textsc{Ilp}
	attains a best or tied-best exact UB on 191 instances and is
	uniquely best on 38; \textsc{Bbc} and \textsc{Lbbd} are uniquely best on
	3 and 2 instances, respectively. On the
	dual side, \textsc{Lbbd} is uniquely best on 80 instances, compared
	with 1 for \textsc{Bbc} and 0 for \textsc{Ilp}.
	
	\subsection{Optimality gap and lower-bound quality}
	\label{sec:exp-gap}
	
	Over the 200 paired feasible instances, a Friedman test detects a significant
	difference among the exact-method gaps ($\chi^2=124.03$,
	$p<10^{-12}$). The aggregate mean gaps are 29.54\%
	for \textsc{Ilp}, 25.97\% for \textsc{Bbc}, and
	24.37\% for \textsc{Lbbd}. On the 11 groups
	for which at least one exact method does not solve all ten instances, the mean
	group gaps are 53.71\%, 47.21\%,
	and 44.30\%, respectively.
	
	\begin{figure}[h]
		\centering
		\includegraphics[width=0.8\textwidth]{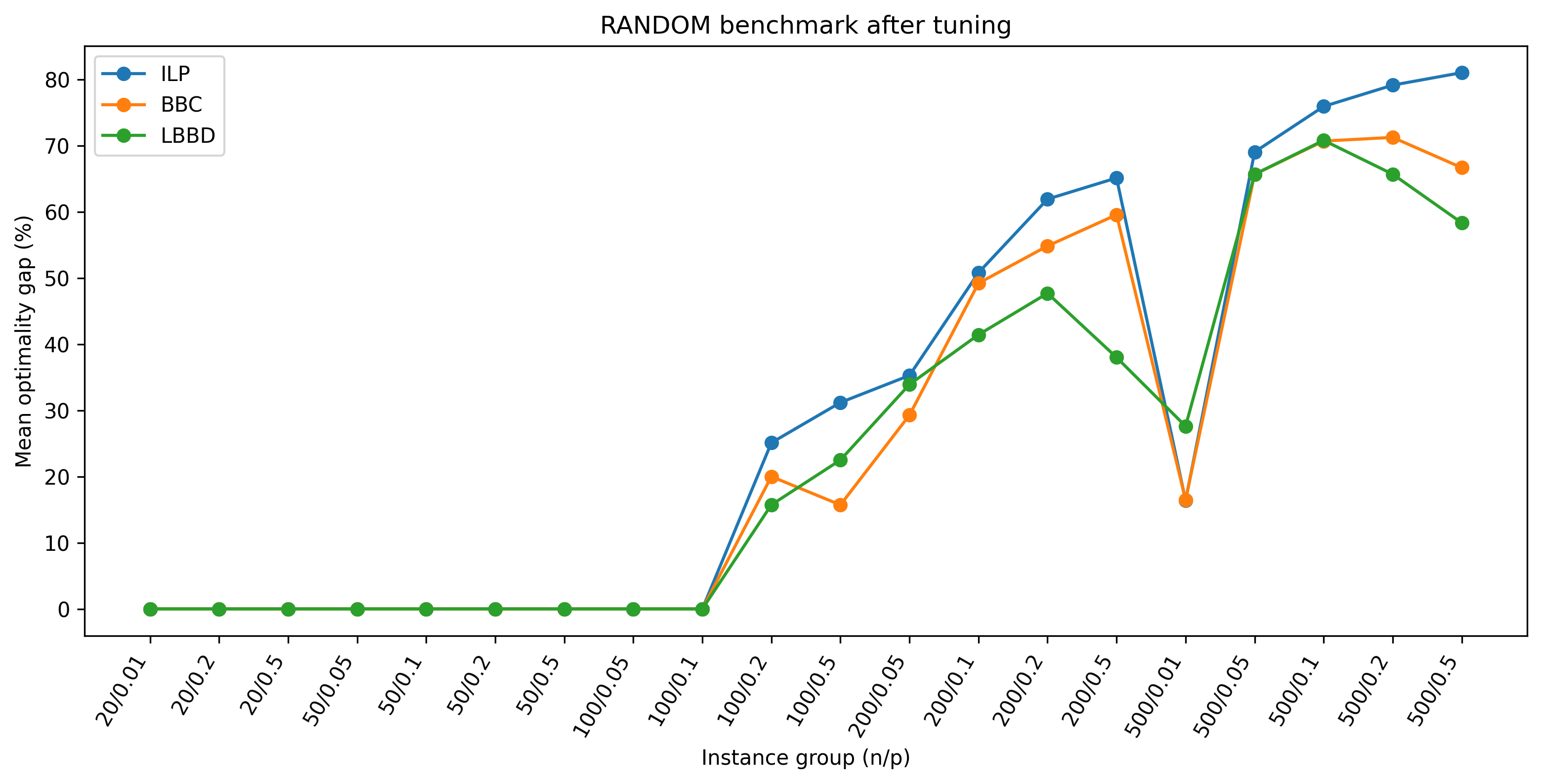}
		\caption{Mean optimality gaps of \textsc{Ilp}, tuned \textsc{Bbc}, and tuned \textsc{Lbbd} on the feasible \textsc{Random} groups.}
		\label{fig:gap-n200}
	\end{figure}

	\begin{table}[!ht]
		\centering
		\caption{Mean lower bounds after tuning. Bold indicates the strongest mean lower bound in each group.}
		\label{tab:lb-comparison}
		\scriptsize
		\begin{tabular}{lrrr}
			\toprule
			Instance ($n/p$) & \textsc{Ilp} & \textsc{Bbc} & \textsc{Lbbd} \\
			\midrule
			20/0.01 & \textbf{9.00} & \textbf{9.00} & \textbf{9.00} \\
			20/0.2 & \textbf{5.00} & \textbf{5.00} & \textbf{5.00} \\
			20/0.5 & \textbf{3.00} & \textbf{3.00} & \textbf{3.00} \\
			50/0.05 & \textbf{16.00} & \textbf{16.00} & \textbf{16.00} \\
			50/0.1 & \textbf{10.00} & \textbf{10.00} & \textbf{10.00} \\
			50/0.2 & \textbf{8.00} & \textbf{8.00} & \textbf{8.00} \\
			50/0.5 & \textbf{6.00} & \textbf{6.00} & \textbf{6.00} \\
			100/0.05 & \textbf{20.00} & \textbf{20.00} & \textbf{20.00} \\
			100/0.1 & \textbf{14.00} & \textbf{14.00} & \textbf{14.00} \\
			100/0.2 & 7.49 & 8.00 & \textbf{9.00} \\
			100/0.5 & 4.82 & \textbf{6.20} & 6.10 \\
			200/0.05 & 18.06 & 20.00 & \textbf{21.00} \\
			200/0.1 & 9.29 & 10.00 & \textbf{12.00} \\
			200/0.2 & 4.91 & 6.00 & \textbf{7.00} \\
			200/0.5 & 3.30 & 4.00 & \textbf{6.00} \\
			500/0.01 & 83.22 & 85.00 & \textbf{86.50} \\
			500/0.05 & 13.41 & \textbf{15.00} & \textbf{15.00} \\
			500/0.1 & 6.55 & \textbf{8.00} & \textbf{8.00} \\
			500/0.2 & 3.63 & 5.00 & \textbf{6.00} \\
			500/0.5 & 2.28 & 4.00 & \textbf{5.00} \\
			\bottomrule
		\end{tabular}
	\end{table}
	
	%\subsection{Optimality certificates}
	
	Using $\lceil\mathrm{LB}\rceil=\mathrm{UB}$ as the certification criterion, \textsc{Bbc} proves
	92 of the 200 feasible instances optimal, compared
	with 90 for \textsc{Ilp} and 91
	for \textsc{Lbbd}.
	
	%\begin{figure}[htbp]
	%	\centering
	%	\includegraphics[width=0.63\textwidth]{figures/random_posttuning_certificates.png}
	%	\caption{Number of feasible \textsc{Random} instances proved optimal after tuning.}
	%	\label{fig:random-certificates}
	%\end{figure}
	
	\subsection{Primal solution quality: exact methods and meta-heuristics}
	\label{sec:exp-grasp}
	
	\textsc{Grasp} and \textsc{Vns} achieve overall mean solution sizes of 19.65 and 19.52,
	respectively, and on the largest instances they frequently provide better
	primal solutions than the incumbents returned by the exact methods within
	the time limit. \\
	On the 200 paired feasible instances, \textsc{Grasp} is better on
	57, \textsc{Vns} on 40, and the methods tie on 103; the paired
	Wilcoxon test remains non-significant ($p=0.414$). Relative to the
	best exact incumbent after tuning, the mean excess is
	0.81\% for \textsc{Grasp} and
	1.41\% for \textsc{Vns}. All numerical results are shown in Table~\ref{tab:ub-comparison}.

	\begin{table}[htbp]
		\centering
		\caption{Mean upper bounds after tuning. Bold indicates the smallest value in each group.}
		\label{tab:ub-comparison}
		\scriptsize
		\begin{tabular}{lrrrrr}
			\toprule
			Instance ($n/p$) & \textsc{Ilp} & \textsc{Bbc} & \textsc{Lbbd} & \textsc{Grasp} & \textsc{Vns} \\
			\midrule
			20/0.01 & \textbf{9.00} & \textbf{9.00} & \textbf{9.00} & \textbf{9.00} & \textbf{9.00} \\
			20/0.2 & \textbf{5.00} & \textbf{5.00} & \textbf{5.00} & \textbf{5.00} & \textbf{5.00} \\
			20/0.5 & \textbf{3.00} & \textbf{3.00} & \textbf{3.00} & \textbf{3.00} & \textbf{3.00} \\
			50/0.05 & \textbf{16.00} & \textbf{16.00} & \textbf{16.00} & \textbf{16.00} & \textbf{16.00} \\
			50/0.1 & \textbf{10.00} & \textbf{10.00} & \textbf{10.00} & \textbf{10.00} & \textbf{10.00} \\
			50/0.2 & \textbf{8.00} & \textbf{8.00} & \textbf{8.00} & \textbf{8.00} & \textbf{8.00} \\
			50/0.5 & \textbf{6.00} & \textbf{6.00} & \textbf{6.00} & \textbf{6.00} & \textbf{6.00} \\
			100/0.05 & \textbf{20.00} & \textbf{20.00} & \textbf{20.00} & 20.63 & 20.01 \\
			100/0.1 & \textbf{14.00} & \textbf{14.00} & \textbf{14.00} & 14.05 & 14.52 \\
			100/0.2 & \textbf{10.00} & \textbf{10.00} & 10.70 & \textbf{10.00} & 10.05 \\
			100/0.5 & \textbf{7.00} & 7.40 & 7.90 & \textbf{7.00} & 7.87 \\
			200/0.05 & 27.90 & 28.30 & 31.80 & 29.33 & \textbf{27.80} \\
			200/0.1 & \textbf{18.90} & 19.70 & 20.50 & 18.97 & 19.94 \\
			200/0.2 & 12.90 & 13.30 & 13.40 & \textbf{12.49} & 12.99 \\
			200/0.5 & 9.50 & 9.90 & 9.70 & \textbf{9.00} & \textbf{9.00} \\
			500/0.01 & \textbf{99.50} & 101.80 & 119.50 & 116.38 & 113.14 \\
			500/0.05 & 43.30 & 43.70 & 43.70 & \textbf{42.40} & 42.99 \\
			500/0.1 & 27.20 & 27.30 & 27.40 & 26.80 & \textbf{26.00} \\
			500/0.2 & 17.40 & 17.40 & 17.50 & \textbf{17.00} & \textbf{17.00} \\
			500/0.5 & 12.00 & 12.00 & 12.00 & \textbf{11.97} & 12.00 \\ \hline 
			AVG.  &  18.83 & 19.09 & 20.26 & 19.65 & 19.52 \\ 
			\bottomrule
		\end{tabular}
	\end{table}
	
	We note that the mean within-instance
	standard deviations of the two heuristic approaches remain 0.21 for \textsc{Grasp} and 0.22 for \textsc{Vns}.
	
	The post-tuning \textsc{Random} results therefore sharpen the division of roles: \textsc{Ilp} is
	the strongest primal exact method, \textsc{Lbbd} gives the strongest dual bounds and
	smallest mean gap, and \textsc{Bbc} proves the largest number of instances optimal.

	\section{NLP Benchmark: Vocabulary Separation for Human- and AI-Generated Text}
	\label{sec:nlp-benchmark}
	
	\subsection{Motivation and graph interpretation}
	\label{sec:nlp-motivation}
	
	The NLP benchmark investigates a feature-selection interpretation of MRBSSP in
	which red and blue vertices represent texts from two classes (human-generated
	and AI-generated), while neutral vertices represent candidate vocabulary
	features. The objective is to identify a smallest set of vocabulary features
	whose graph-incidence patterns distinguish every red--blue pair. %All red and
	%blue object vertices are completed to a clique while neutral feature vertices
	%and their original incidences are left unchanged. 
	The separator here is forced to rely on neutral vocabulary features.
	
	\paragraph{Relation to work on machine-generated text.}
	Detection of machine-generated text is an active research area; a recent
	survey is given by Wu et al.~\cite{WuEtAl2025}. Many detectors assign a
	score to a document, either through a trained classifier or through
	statistics of token probabilities under a language
	model~\cite{GehrmannEtAl2019,MitchellEtAl2023}, and their decisions are not
	stated in terms of a small set of observable features; detectors that do
	rely on a small set of stylometric or lexical features do not ask whether
	that set is the smallest one that separates the sample. Studies that compare
	human-written and machine-generated texts directly report systematic
	differences in vocabulary and style. Herbold et
	al.~\cite{HerboldEtAl2023} compare essays written by students and by
	ChatGPT with respect to lexical diversity, syntactic complexity and the use
	of modal, epistemic and discourse markers, and Kobak et
	al.~\cite{KobakEtAl2025} identify words whose frequency in biomedical
	abstracts increased sharply after large language models became widely
	available. In both cases the differences are established by frequency
	statistics over large collections of texts. In text classification,
	feature selection methods rank words by criteria such as document
	frequency, the $\chi^2$ statistic or information
	gain~\cite{YangPedersen1997,Forman2003}, and the number of retained words
	is left to the user. The MRBSSP formulation asks a different question: what
	is the smallest set of words whose presence or absence distinguishes every
	human-written text of a given sample from every machine-generated text of
	the same sample? In rough set theory this object is a minimum decision re-duct of the
	binary decision table whose objects are the texts, whose attributes are
	the candidate words and whose decision is the class~\cite{SkowronRauszer1992}, and the exact methods
	of this paper make it possible to certify that no smaller set exists. The
	benchmark is not intended as a new detector. Separating all pairs of a
	fixed sample is a property of that sample, and the classification of unseen texts is not evaluated here.
	%% ---- DOPUNA ----
	{In practice, a certified minimum separator is a tool for inspecting a data set rather than for labelling new texts. Its size is a lower bound on the number of words that any rule based on word presence needs in order to tell the two sources apart on that sample, so it measures how far apart the two sources are lexically and can be compared across samples, domains or generators. Its content shows which words suffice for the separation; when these are a few function words or topic words, as in Section~\ref{sec:nlp-optimal-node-interpretation}, a detector trained on the same sample can separate it perfectly through surface features unrelated to authorship, which is the kind of shortcut that benchmark construction for machine-text detection tries to avoid.}
	%% ---- kraj dopune -----  
	
	\paragraph{Corpus and sampling.}
	The instances are built from the English part of the HC3 corpus (Human
	ChatGPT Comparison Corpus)~\cite{GuoEtAl2023}. HC3 contains 24\,322
	questions from five sources (Reddit ELI5, finance, medicine, open-domain
	question answering and Wikipedia-based computer science), together with
	58\,546 human-written answers and 26\,903 answers generated by ChatGPT.
	Every question has at least one human answer; most questions have one
	ChatGPT answer, 1\,833 have several and 455 have none. An instance is
	determined by the number $\qns$ of questions and by a seed. The $\qns$ questions
	are drawn uniformly at random without replacement from the whole corpus,
	using a pseudo-random number generator initialized with the seed, and for
	every drawn question the first human answer and all ChatGPT answers are
	retained.
	The seeds are the consecutive integers 42 to 51 and were fixed before any
	experiment was run, so that no sample was selected or discarded on the
	basis of the results. Since the generator draws the questions one at a
	time, the sample of $\qns$ questions for a given seed consists of the first
	$\qns$ questions of every larger sample with the same seed, and the samples
	with the same seed are therefore nested across sizes.
	Each answer is one document. The human answers form the red set
	$R$ and the ChatGPT answers form the blue set $B$, so that $|R|=\qns$ before
	preprocessing, while $|B|$ lies between $1.06\,\qns$ and $1.26\,\qns$ in our
	samples because some questions contribute more than one generated answer.
	
	\paragraph{Tokenization and binarization.}
	Every document is first converted to lower case. HC3 is distributed in
	Penn Treebank tokenization~\cite{marcus1993building}, in which clitics appear as separate tokens
	(for example \emph{should n't} or \emph{they 're}), so the clitics are
	normalized before words are extracted: the truncated forms \emph{ca n't},
	\emph{wo n't} and \emph{sha n't} are expanded to \emph{can not},
	\emph{will not} and \emph{shall not}, every remaining \emph{n't} is
	rewritten as \emph{not}, and the other clitics (\emph{'s}, \emph{'re},
	\emph{'ve}, \emph{'ll}, \emph{'d}, \emph{'m}) are removed because their
	expansion is ambiguous. The text is then split into words, a word being a
	maximal sequence of letters and digits, and one-letter tokens other than
	\emph{a} and \emph{i} are discarded. The token \emph{url}, which HC3
	substitutes for web links in the human answers, is excluded, since it would
	separate pairs for a reason unrelated to the language of the texts. For a threshold $k\ge 1$, a word $t$ is considered present in a document $d$ if
	it occurs in $d$ at least $k$ times, and the binary incidence matrix is
	defined by $A_{dt}=1$ if $t$ is present in $d$ and $A_{dt}=0$ otherwise. A
	word is kept as a candidate if it is present in at least two documents and
	in at most $90\,\%$ of the documents of the instance, both counts being
	taken on the binarized matrix. The standard instances use $k=1$, that is,
	plain presence. The variants use $k=2$ or $k=3$ on the same sample of questions as the corresponding standard instance.
	%% ---- DOPUNA ----
	{A threshold $k>1$ counts a word only when a text uses it repeatedly (at least $k$ times), so a single incidental occurrence no longer separates a pair. It also yields instances of a different structure, with three to eight times fewer candidates and some pairs separated by very few words, which is why the variants behave differently from the standard instances in Sections~\ref{sec:nlp_pace} and~\ref{sec:nlp-optimal-node-interpretation}.}
	%% ---- kraj dopune ----

	\paragraph{Graph.}
	Every document is a colored vertex, red for $R$ and blue for $B$, and
	every candidate word is a neutral vertex. A word vertex $t$ is adjacent to
	a document vertex $d$ if and only if $A_{dt}=1$, and there are no edges
	between word vertices. The document vertices are then completed to a
	clique, as described above, so that no document vertex separates any pair.
	A word $t$ separates the pair $(r,b)$ exactly when it is present in one of
	the two documents and absent from the other, and the covering constraints
	of the instance are given by the sets
	\[
	D(r,b)=\{\,t : A_{rt}\neq A_{bt}\,\},\qquad r\in R,\ b\in B .
	\]
	Completing the documents to a clique and restricting the candidate set to
	the word vertices yield the same sets $D(r,b)$, because the clique changes
	only the neighborhoods of the document vertices. %The instance files
	%contain only the document--word edges and the colours of the vertices; the
	%clique is added when an instance is loaded.
	
	\paragraph{Inseparable pairs.}
	If a red and a blue document have identical rows of $A$, then
	$D(r,b)=\emptyset$ and the instance has no feasible separator. This does
	not happen on any standard instance, but it does happen under the stricter
	binarization, because a short text retains only a few words when $k\ge 2$
	and two different texts can then have the same row. Such pairs are
	detected by grouping the documents by their rows, and for every row that
	is shared by documents of both classes the documents of the smaller class
	are removed. Groups with different rows are independent, so this removes
	the smallest possible number of documents.
	%% ---- DOPUNA ----
	{This step removes nothing from the standard instances and between zero and four documents on each variant, out of more than 200 documents at $\qns=100$ and more than 500 at $\qns=250$, so the samples remain practically unchanged.}
	%% ---- kraj dopune ----

	\paragraph{Optimum-preserving preprocessing.}
	Before an instance is written to file, two reductions that do not change
	the optimum value are applied: documents of the same class with identical
	rows of $A$ are merged, and a word $t$ is removed when every pair it
	separates is also separated by another word $t'$. On vocabularies larger
	than 3\,000 words the second test is applied to a random sample of word
	pairs, so that some dominated words may remain. Neither reduction touches
	the pair constraints, so the kernelization of Section~\ref{sec:ilp}, which
	removes dominated rows and fixes essential words, is still applied by the
	exact methods on top of them.
	
	\paragraph{Instance set.}
	An instance is denoted by the triple $(\qns, \mathit{seed}, k)$; for
	example, $(100, 42, 3)$ is the variant with $k=3$ on the sample of 100
	questions drawn with seed 42. In the published files this instance is
	named \texttt{hc3\_q100\_seed42\_k3}, and the names of the standard
	instances carry no suffix.
	The benchmark contains 42 instances. The 32 standard instances use $k=1$,
	with $\qns\in\{100,150,200\}$ and seeds 42 to 51, and with $\qns=250$ and seeds
	42 and 43. The ten variants use $k=2$ and $k=3$ on the samples with
	$\qns=100$ and seeds 42, 43 and 44, and on the samples with $\qns=250$ and seeds
	42 and 43. Table~\ref{tab:nlp-instances} summarizes the size of the
	instances after preprocessing. The number of red--blue pairs grows
	quadratically with $\qns$, from about $1.1\cdot 10^{4}$ at $\qns=100$ to about
	$7\cdot 10^{4}$ at $\qns=250$, whereas the number of candidate words grows
	roughly linearly. The stricter binarization reduces the candidate set by a
	factor of three to eight and leaves some pairs that are separated by very
	few words; on six of the ten variants some pair is separated by a single
	word, which then belongs to every feasible separator.
	
	\begin{table}[htbp]
		\centering
		\caption{Size of the 42 NLP instances after preprocessing, given as the
			range over the instances of each group. $|T|$: number of candidate
			words; $|E|$: number of document--word edges; $|R|\,|B|$: number of
			red--blue pairs.}
		\label{tab:nlp-instances}
		\small
		\begin{tabular}{llrrrrrr}
			\toprule
			Group & $\qns$ & $N$ & $|R|$ & $|B|$ & $|T|$ & $|E|$ & $|R|\,|B|$ \\
			\midrule
			standard, $k=1$ & 100 & 10 & 99--100 & 108--126 & 1\,622--1\,829 & 13\,239--14\,686 & 10\,800--12\,600 \\
			& 150 & 10 & 149--150 & 159--188 & 2\,325--2\,500 & 21\,313--23\,043 & 23\,850--28\,200 \\
			& 200 & 10 & 199--200 & 213--244 & 2\,880--3\,163 & 29\,307--31\,342 & 42\,600--48\,800 \\
			& 250 & 2  & 250 & 279--291 & 3\,526--3\,621 & 37\,265--38\,761 & 69\,750--72\,750 \\
			\midrule
			variant, $k=2$  & 100 & 3  & 93--95 & 108--116 & 477--568 & 4\,397--5\,042 & 10\,260--10\,788 \\
			& 250 & 2  & 231--236 & 279--291 & 1\,080--1\,141 & 11\,743--12\,527 & 65\,844--67\,221 \\
			variant, $k=3$  & 100 & 3  & 75--82 & 108--115 & 222--267 & 2\,175--2\,618 & 8\,100--8\,938 \\
			& 250 & 2  & 185--202 & 276--287 & 486--541 & 5\,970--6\,321 & 53\,095--55\,752 \\
			\bottomrule
		\end{tabular}
	\end{table}

	\subsubsection{Sensitivity of the $q_{\text{BBC}}$ and $q_{\text{LBBD}}$ parameters }
	\label{sec:tuning-nlp}
	
	The complete NLP tuning experiment evaluates
	\[
	q_{\mathrm{BBC}},q_{\mathrm{LBBD}}\in
	\{1,2,5,10,50,100,500\}
	\]
	on all 42 NLP benchmark instances: the 32 standard
	instances together with the ten $k \in \{2, 3\}$ variants. These results are reported in Appendix~\ref{sec:appendix-nlp-tuning}. We report mean primal and dual bounds, mean relative optimality gap, number of proven optima, and separation effort over all 42 instances. We selected $q_{\text{BBC}}$ and $q_{\text{LBBD}}$ by maximizing the number of proven optima, using the mean relative optimality gap to break ties.

	The complete \textsc{Bbc} experiment confirms that a small cut batch is preferable.
	$q_{\mathrm{BBC}}=2$ gives the smallest mean gap,
	13.21\%, and proves
	16 of the 42
	instances optimally. The only other setting with the same number of certificates
	is $q_{\text{BBC}}=10$, but its mean gap is slightly larger
	(13.29\%) and its mean
	lower bound is slightly weaker
	(12.05 versus
	12.12). The 
	$q_{\text{BBC}}=500$ setting proves
	12 instances
	optimal, but its mean gap rises to
	16.26\%, indicating
	that very aggressive callback batching is not beneficial for \textsc{Bbc}. We therefore
	select
	\[
	\boxed{q_{\mathrm{BBC}}=2}.
	\]

	For \textsc{Lbbd}, $q_{\text{LBBD}}=50$ gives the smallest mean gap
	(14.20\%), whereas
	$q_{\text{LBBD}}=100$ is very close
	(14.38\%).
	Both values attain the same mean lower bound,
	12.29, but
	$q_{\text{LBBD}}=100$ is the only setting that proves 11 instances optimal and it also
	reduces the average number of master iterations to 18.4.
	The $q_{\text{LBBD}}=500$ setting reduces the iteration count further to
	11.8, but its mean
	gap deteriorates to 15.23\% and the number of certificates drops to
	10. Hence, we retain
	\[
	\boxed{q_{\mathrm{LBBD}}=100}.
	\]
	
	Figures~\ref{fig:tuning-nlp-gap}--\ref{fig:tuning-nlp-effort} summarize the
	complete tuning study.
	
	\begin{figure}[h]
		\centering
		
		\begin{subfigure}[t]{0.48\textwidth}
			\centering
			\includegraphics[width=\textwidth]{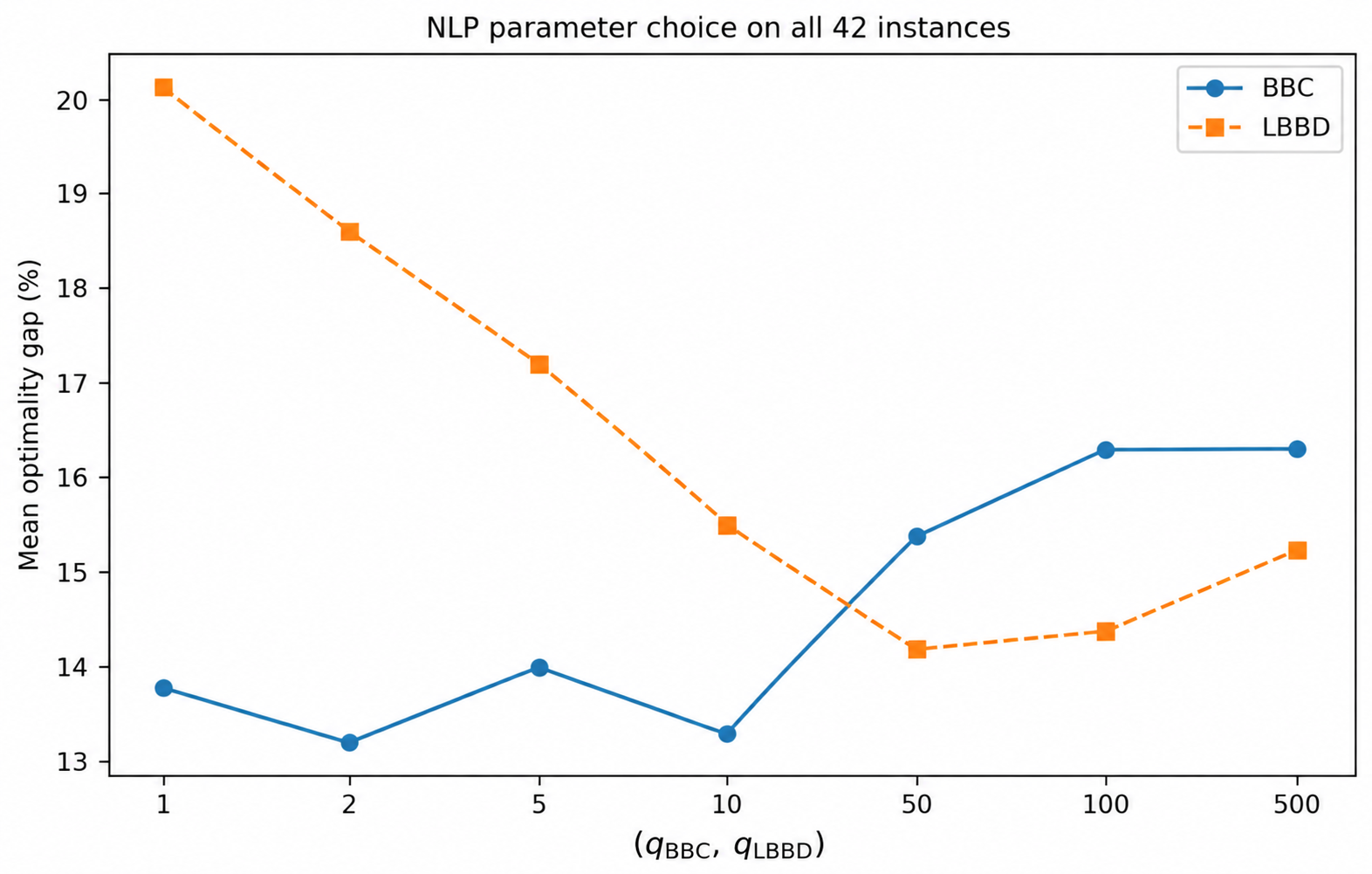}
			\caption{Mean optimality gap for \textsc{Bbc} and \textsc{Lbbd} over all 42 NLP instances.}
			\label{fig:tuning-nlp-gap}
		\end{subfigure}
		\hfill
		\begin{subfigure}[t]{0.48\textwidth}
			\centering
			\includegraphics[width=\textwidth]{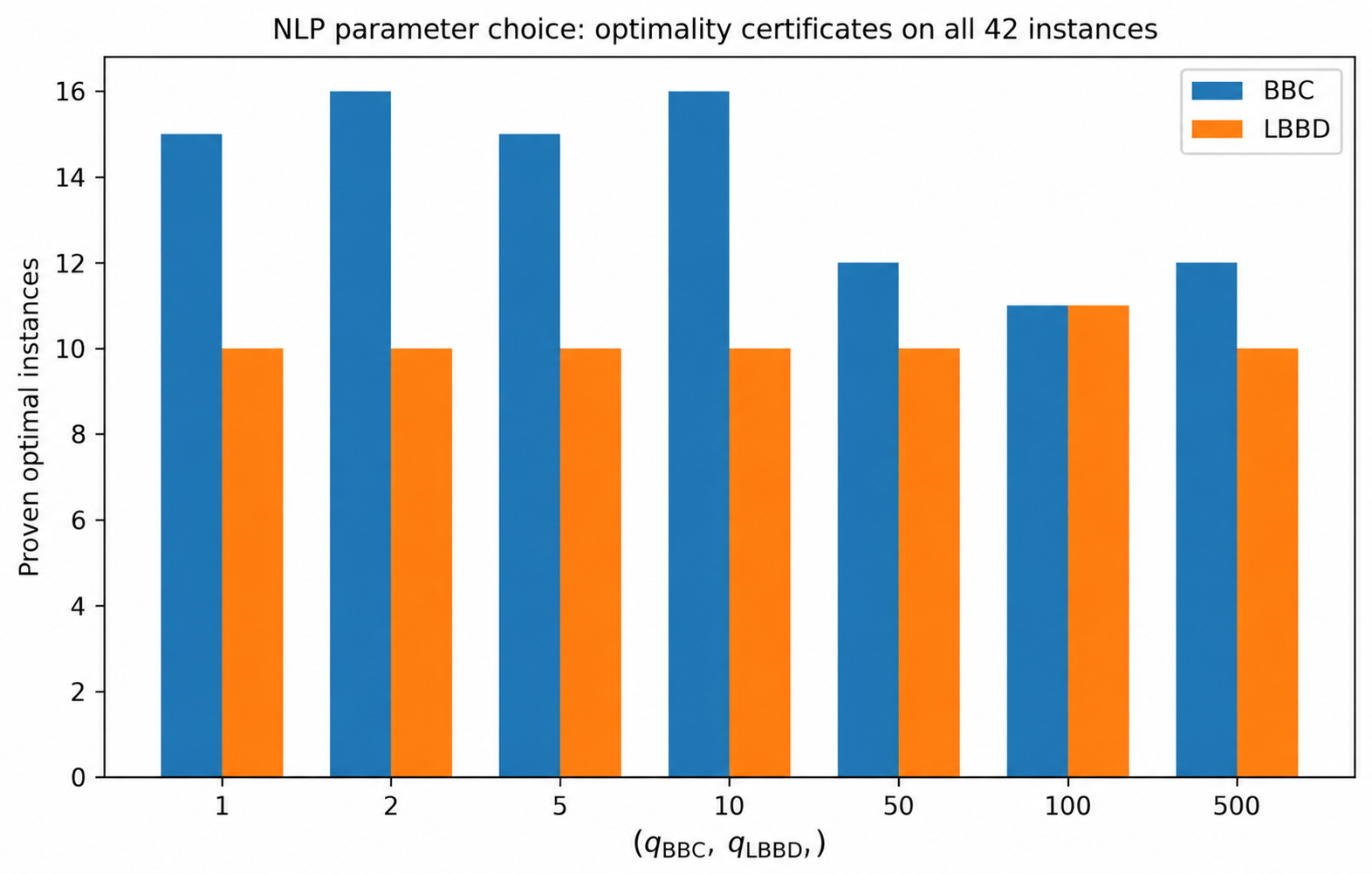}
			\caption{Number of the 42 NLP instances proved optimal for each tested parameter value.}
			\label{fig:tuning-nlp-certificates}
		\end{subfigure}
		
		\vspace{0.5cm}
		
		\begin{subfigure}[t]{0.48\textwidth}
			\centering
			\includegraphics[width=\textwidth]{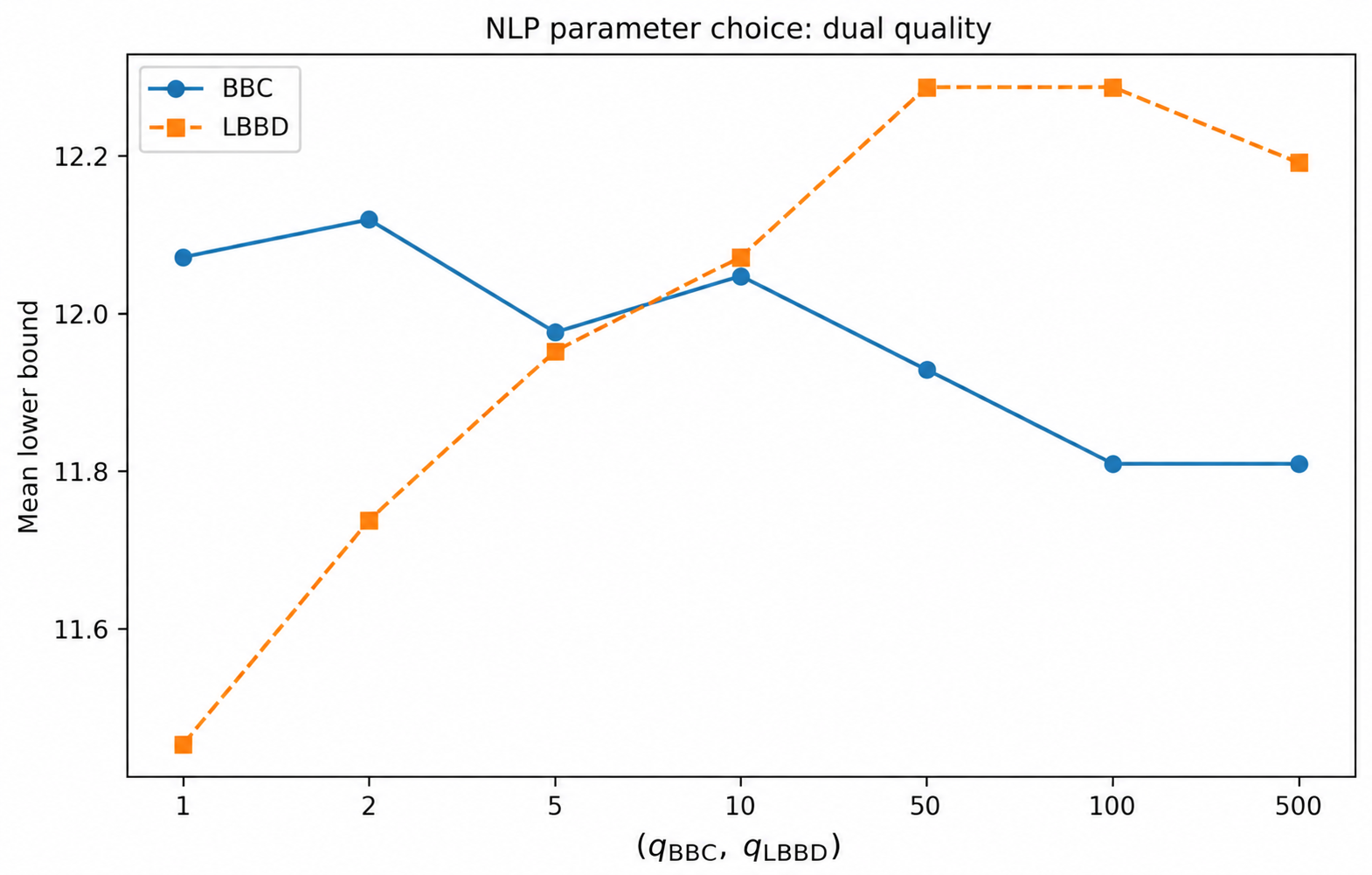}
			\caption{Mean LBs over all 42 NLP instances.}
			\label{fig:tuning-nlp-lb}
		\end{subfigure}
		\hfill
		\begin{subfigure}[t]{0.48\textwidth}
			\centering
			\includegraphics[width=\textwidth]{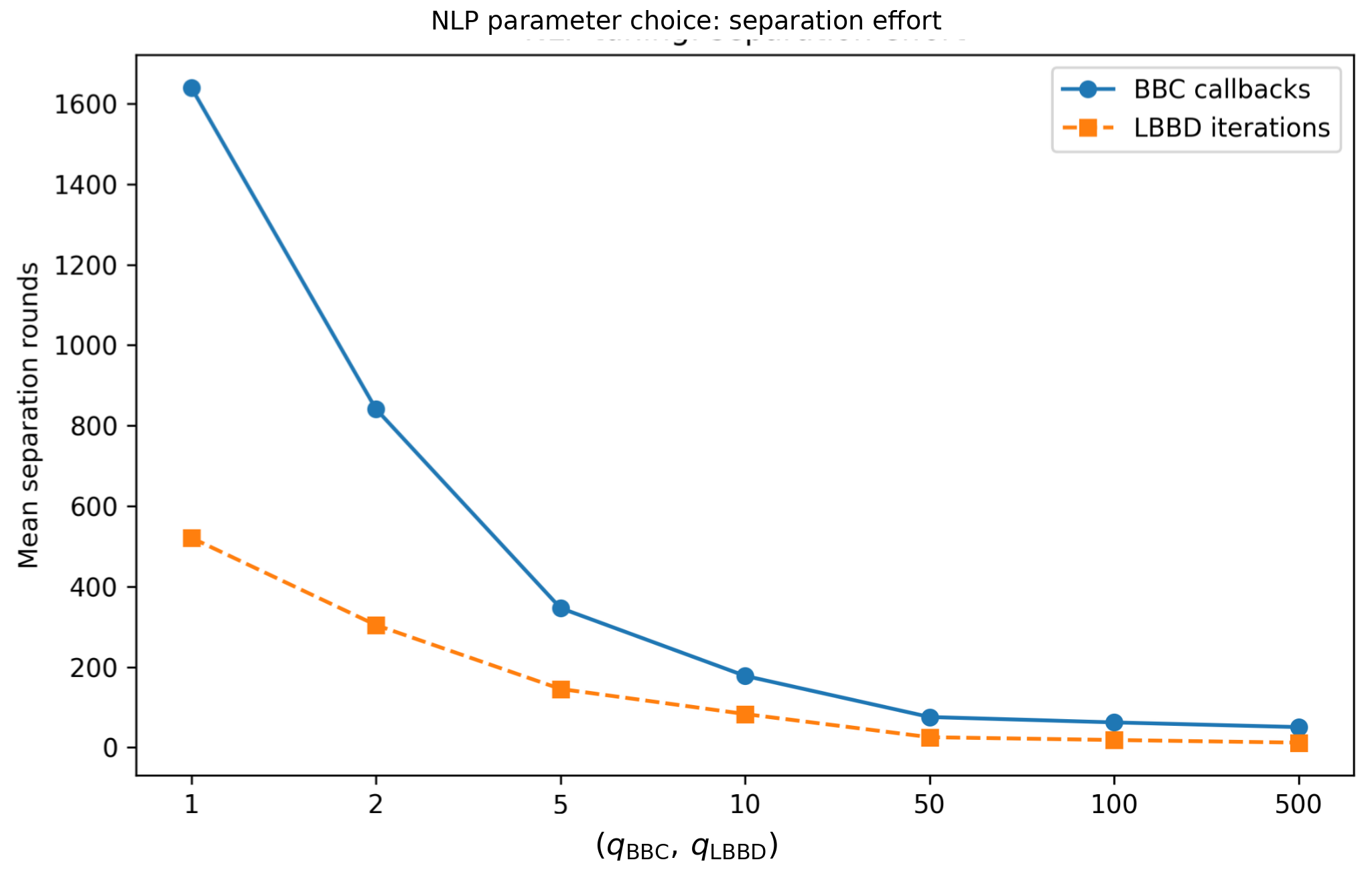}
			\caption{Mean separation effort (\# of iterations) during NLP tuning: \textsc{Bbc} callback invocations and \textsc{Lbbd} master iterations.}
			\label{fig:tuning-nlp-effort}
		\end{subfigure}
		
		\caption{NLP tuning: summary results for \textsc{Bbc} and \textsc{Lbbd} over all 42 instances.}
		\label{fig:tuning-nlp-complete}
	\end{figure}
	
	\subsection{Experimental protocol and coverage}
	\label{sec:nlp-protocol}
	
	%The NLP benchmark contains 42 instances in total: 32 standard instances
	%(ten for each of $q\in\{100,150,200\}$ and two for $q=250$) together with
	%ten $k=2/3$ variants. 
	Following the complete tuning study, BBC is run with
	$q_{\mathrm{BBC}}=2$ and LBBD with $q_{\mathrm{LBBD}}=100$.
	ILP, GRASP, and VNS follow the same protocol used for the \textsc{Random} benchmark.
	All aggregate comparisons below use all 42 instances unless stated otherwise.
	
	\subsection{Overall performance across all approaches}
	\label{sec:nlp-exact-results}
	
	Table~\ref{tab:nlp-overall42} reports the aggregate results over the complete
	42-instance NLP benchmark.
	
	\begin{table}[!htb]
		\centering
		\caption{Overall NLP results on all 42 benchmark instances. Exact methods
			report mean UB, mean LB, mean relative gap, and the number of proven optima.
			\textsc{Grasp} and \textsc{Vns} report the per-instance ten-run mean solution size.}
		\label{tab:nlp-overall42}
		\small
		\begin{tabular}{lrrrr}
			\toprule
			Method & Mean UB & Mean LB & Mean gap (\%) & Optimal \\
			\midrule
			\textsc{Ilp} & \textbf{13.62} & 10.55 & 23.75 & 11/42 \\
			\textsc{Bbc} & 13.90 & \textbf{12.12} & \textbf{13.21} & \textbf{16/42} \\
			\textsc{Lbbd} & 14.26 & 12.29 & 14.38 & 11/42 \\
			\textsc{Grasp} & 13.73 & -- & -- & -- \\
			\textsc{Vns} & 14.30 & -- & -- & -- \\
			\bottomrule
		\end{tabular}
	\end{table}
	
	Across all 42 instances, \textsc{Ilp} remains the strongest exact method on the primal
	side, with mean UB 13.62. \textsc{Bbc} follows with
	13.90, while \textsc{Lbbd} has mean UB
	14.26. In contrast, the decomposition methods
	substantially strengthen the dual side: BBC attains mean LB
	12.12 and \textsc{Lbbd} reaches the strongest mean LB,
	12.29, compared with
	10.55 for \textsc{Ilp}. Consequently, BBC has the
	smallest overall mean optimality gap,
	13.21\%, followed by \textsc{Lbbd} at
	14.38\%, whereas \textsc{Ilp} remains at
	23.75\%.
	
	The certificate comparison is similarly favorable to \textsc{Bbc} : it proves
	16 of the 42 instances optimal,
	compared with 11 for \textsc{Ilp} and \textsc{Lbbd} each. A Friedman test over the
	42 paired exact-method gaps confirms a significant difference
	($\chi^2=42.05$,
	$p=7.39\times 10^{-10}$).
	
	%\begin{figure}[htbp]
	%	\centering
	%	\includegraphics[width=0.65\textwidth]{figures/nlp_all42_gap.png}
	%	\caption{Mean optimality gap of the exact methods over all 42 NLP instances.}
	%	\label{fig:nlp-gap-all42}
	%\end{figure}
	
	%\begin{figure}[htbp]
	%	\centering
	%	\includegraphics[width=0.65\textwidth]{figures/nlp_all42_lb.png}
	%	\caption{Mean certified lower bound of the exact methods over all 42 NLP instances.}
	%	\label{fig:nlp-lb-all42}
	%\end{figure}
	
	Table~\ref{tab:nlp-grouped-results} gives the grouped view over the complete
	42-instance benchmark. The ten $k \in \{2, 3\}$ variants are included in their
	corresponding NLP-$\qns$ groups; consequently, the number of instances per group are $16$ for
	$\qns=100$, $10$ for $\qns=150$, $10$ for $\qns=200$, and $6$ for $\qns=250$.
	
	\begin{table}[htbp]
		\centering
		\caption{Grouped results on all 42 NLP benchmark instances. The $N$ column
			shows the number of instances contributing to each NLP-$\qns$ group. Exact methods
			report mean UB, mean LB, and mean gap; heuristics report ten-run means.}
		\label{tab:nlp-grouped-results}
		\scriptsize
		\begin{tabular}{cr|rrr|rrr|rrr|rr}
			\toprule
			& & \multicolumn{3}{c|}{ILP} & \multicolumn{3}{c|}{BBC} &
			\multicolumn{3}{c|}{LBBD} & \multicolumn{2}{c}{Heuristics}\\
			$\qns$ & $N$ & UB & LB & Gap (\%) & UB & LB & Gap (\%) &
			UB & LB & Gap (\%) & GRASP & VNS\\
			\midrule
			100 & 16 & \textbf{11.81} & 10.39 & 13.12 & 12.00 & \textbf{11.44} & \textbf{4.92} & 12.31 & 11.38 & 7.95 & 11.88 & 12.41 \\
			150 & 10 & {12.80} & 8.60 & 32.89 & 13.30 & 10.60 & \textbf{20.38} & 13.80 & \textbf{10.80} & 21.83 & \textbf{12.75} & 13.49 \\
			200 & 10 & \textbf{13.90} & 8.69 & 37.62 & 14.30 & 11.10 & 22.40 & 14.70 & \textbf{11.60} & \textbf{21.15} & 14.10 & 14.56 \\
			250 & 6 & \textbf{19.33} & 17.33 & 13.74 & \textbf{19.33} & 18.17 & 8.02 & 19.50 & \textbf{18.33} & \textbf{7.78} & 19.68 & 20.28 \\
			\bottomrule
		\end{tabular}
	\end{table}

	\subsection{Primal heuristic behavior}
	\label{sec:nlp-primal-results}
	
	Across all 42 instances, \textsc{Grasp} obtains mean separator size
	13.73, compared with
	14.30 for \textsc{Vns}. At the instance level, \textsc{Grasp} has
	the smaller ten-run mean on 31 instances, \textsc{Vns} on 2, and the methods tie on
	9. The paired Wilcoxon signed-rank test gives
	$p=7.80\times 10^{-7}$.
	Relative to the best exact incumbent on each instance, the average excess is
	0.76\% for \textsc{Grasp} and
	5.21\% for \textsc{Vns}.
	
	%\begin{figure}[htbp]
	%	\centering
	%	\includegraphics[width=0.72\textwidth]{figures/nlp_all42_ub.png}
	%	\caption{Mean primal solution size for all five approaches over the complete
		%		42-instance NLP benchmark.}
	%	\label{fig:nlp-ub-all42}
	%\end{figure}
	
	\subsection{Optimality certificates and dynamic constraint generation}
	\label{sec:nlp-dual-results}
	
	Because MRBSSP has an integer-valued objective, optimality is counted whenever
	$\lceil LB\rceil=UB$. On the complete 42-instance benchmark, \textsc{Bbc} solves
	16 instances to proven optimality, \textsc{Ilp} and \textsc{Lbbd} prove the optimal solutions on 11 instances each.

	The selected cut-batch values also lead to markedly different separation
	dynamics. With $q_{\mathrm{BBC}}=2$, \textsc{Bbc} performs on average
	841.5 callback invocations and adds
	1679 lazy constraints over the 42 instances. With
	$q_{\mathrm{LBBD}}=100$, \textsc{Lbbd} requires only
	18.4 master iterations on average while adding
	1145 cuts. These results indicate the architectural
	difference between the two methods: \textsc{Bbc} benefits from selective cut insertion
	inside one branch-and-bound tree, whereas \textsc{Lbbd} benefits from larger batches
	that reduce repeated master re-optimizations.

	\subsection{Comparison with a general-purpose hitting set solver}
	\label{sec:nlp_pace}
	
	Once the candidate set is fixed, an NLP instance is a hitting set problem~\cite{chandrasekaran2011algorithms,Gainer2017minimal}. 
	The elements are the candidate words, and every human/ChatGPT pair $(r,b)$
	contributes the set $D(r,b)$ of candidates that occur in exactly one of the two texts. A set of words separates all pairs if and only if it meets every
	$D(r,b)$, so a minimum hitting set of this family is a minimum separator and
	any exact hitting set solver must reach the same optimum as \textsc{Ilp}, \textsc{Bbc} and
	\textsc{Lbbd}. This gives a direct way to measure how much the problem-specific
	machinery contributes. We converted every instance into the input format of the
	PACE 2025 challenge, listing all $|R|\,|B|$ sets explicitly and without any
	reduction of our own, and ran UzL~\cite{BannachEtAl2025}, which was ranked
	first in the exact hitting set track of PACE 2025~\cite{GroblerSiebertz2025}. UzL
	first shrinks the hypergraph with four standard reduction rules (a set of
	size one forces its element, an element contained in a single set is
	dropped, a set that contains another set is dropped, and an element whose
	sets all contain a second element is dropped) and then solves a partial
	MaxSAT formulation with EvalMaxSAT~\cite{Avellaneda2020}. It outputs a solution only after proving
	its optimality, so a run that reaches the time limit yields neither a solution
	nor a bound. 
	%% ---- DOPUNA ----
	{The time limit of a UzL job covers the whole job, including reading the input and the reductions, whereas the 3600\,s of the exact methods are counted only after their preprocessing and warm start. To establish a more rigorous comparison,  we therefore ran every instance twice, with job limits of 3600\,s and 7200\,s; the second run gives UzL about as much time for the search itself as the exact methods had.}
	%% ---- kraj dopune ----  %, whereas the 3600\,s of the exact methods start after
	%their preprocessing and warm start. 
	Each UzL job ran as a single process on one core of a compute node of the
	same cluster, which matches the single-threaded setting used for all other
	methods in Section~\ref{sec:experiments}.
	
	%% ---- DOPUNA ----
	{\begin{table}[htbp]
			
			\centering
			\caption{UzL and \textsc{Bbc} on the three groups of NLP instances. $|T|$: number of candidate words;
				$|D|$: smallest and mean size of the sets $D(r,b)$; both are given as ranges over the
				instances of the group. Single: instances with a pair separated by one word. UzL: instances whose optimum
				UzL is able to prove and the time of the shorter of the two runs, which covers the whole job.
				\textsc{Bbc}: instances certified under the one-hour limit and their solver time.}
			\label{tab:uzl-groups}
			\setlength{\tabcolsep}{4pt}
			\footnotesize
			\begin{tabular}{lrrrrrrrrr}
				\toprule
				& & & \multicolumn{2}{c}{$|D|$} & & \multicolumn{2}{c}{UzL} & \multicolumn{2}{c}{\textsc{Bbc}} \\
				\cmidrule(lr){4-5}\cmidrule(lr){7-8}\cmidrule(lr){9-10}
				Group & $N$ & $|T|$ & min & mean & Single & proved & time (s) & proved & time (s) \\
				\midrule
				standard, $k=1$ & 32 & 1\,622--3\,621 & 3--26 & 94--121 & 0 & 0 & -- & 6 & 333--2\,680 \\
				variant, $k=2$  & 5  & 477--1\,141   & 1--5  & 30--36  & 3 & 4 & 73--1\,378 & 5 & 4--20 \\
				variant, $k=3$  & 5  & 222--541      & 1--2  & 16--19  & 3 & 5 & 5--30 & 5 & 0--1 \\
				\bottomrule
			\end{tabular}
		\end{table}
		
		Table~\ref{tab:uzl-groups} shows that the outcome differs between the two kinds of instances in the benchmark.
		UzL proved the optimum on all 10 variants ($k\in\{2,3\}$) and on none of the 32 standard instances ($k=1$); the
		longer time limit did not close a single standard instance. The $k=2$ variants on $\qns=250$ are harder. On
		$(250, 43, 2)$ UzL did not finish within either limit on the cluster, where several jobs shared a compute node.
		A separate run on a workstation with no other load proved the same optimum as \textsc{Bbc} in 4540\,s; that run
		only confirms the value and is not used in the comparison of running times. Every solution UzL returned separates
		all pairs, has the size certified by \textsc{Bbc} and contains all words that the forbid-one-word protocol of
		Section~\ref{sec:nlp-optimal-node-interpretation} places in the core; on $(250, 43, 2)$, whose optimum is unique,
		it returns exactly that dictionary. \textsc{Bbc} closes the same instance in 20\,s of solver time, after about
		2300\,s of preprocessing. Under the one-hour time limit, \textsc{Bbc} certifies 16 of the 42 instances
		(Table~\ref{tab:nlp-overall42}), among them six standard instances on which UzL does not finish even with the
		doubled time limit.
		
		The difference follows from the structure of the sets. On the standard instances the sets $D(r,b)$ are long and
		no pair is separated by fewer than three words, so the kernelization of Section~\ref{sec:ilp}, which uses rules
		of the same kind as UzL (essential vertices, row dominance and column dominance), finds no essential word and
		removes at most 0.7\,\% of the pair sets. On the variants it removes between 13 and 94\,\% of them, and on six of
		the ten a pair is separated by a single word, which the first reduction rule fixes at once, so the instance
		shrinks considerably before any search starts. On the standard instances reductions of this kind can be expected
		to leave the MaxSAT solver with nearly the full family of 10\,800 to 72\,750 long clauses. \textsc{Bbc}, by
		contrast, never needs most of these constraints: on the standard instances it adds as lazy cuts between 3.5 and
		17.6\,\% of the pair constraints that remain after kernelization.}
	%% ---- kraj dopune ----

	\subsection{Structural Interpretation of Optimal Vocabulary Separators}
	\label{sec:nlp-optimal-node-interpretation}
	
	Beyond optimum cardinality, the NLP instances allow us to study the
	\emph{structure of the optimal solution space}. For this structural analysis, we reran \textsc{Bbc} with the tuned setting
	$q_{\mathrm{BBC}}=2$ and a 7200-second time limit to determine whether the longer run would certify additional instances. In this experiment, 19 instances were proved optimal. \\
	
	Let $S^\star$ be one certified
	optimal separator of cardinality $z^\star$. For every $v\in S^\star$, we solve
	the restricted problem with the additional constraint
	\[
	x_v=0,
	\]
	that is, the corresponding word is \textit{forbidden} from the separator. We classify
	the tested word according to the best primal and dual information obtained for
	the restricted problem.
	
	\begin{itemize}
		\item \textbf{Core/mandatory word.}
		If the restricted problem is infeasible, or its certified integer lower
		bound satisfies
		\[
		\lceil LB_{-v}\rceil > z^\star,
		\]
		then no optimum of cardinality $z^\star$ can exclude $v$. Hence $v$ belongs
		to every optimal separator and is part of the \emph{optimal core}.
		Among such words, we distinguish \emph{objective-critical} words, for which
		the restricted optimum is strictly larger than $z^\star$, from
		\emph{feasibility-critical} words, for which the restricted model is
		reported infeasible.
		
		\item \textbf{Replaceable word.}
		If the restricted run finds a feasible separator of size $z^\star$, then
		another optimum exists without $v$. Thus, $v$ occurs in the reference
		optimum but not in every optimum.
		
		\item \textbf{Restrictive--inconclusive word.}
		If no alternative optimum of size $z^\star$ is found, but
		$\lceil LB_{-v}\rceil\le z^\star$, then the run suggests that forbidding
		$v$ is restrictive, but it does not prove that $v$ belongs to the core.
		This category is therefore kept separate from the proven core.
	\end{itemize}
	
	The structural experiment covers all 19 NLP instances for which \textsc{Bbc} proved
	optimality when run with $q_{\mathrm{BBC}}=2$ and a 7200~s time limit.
	For each instance, one word at a time is forbidden from the corresponding
	certified optimum, resulting in 272 restricted \textsc{Bbc} runs. Each run uses a
	single compute node with the same resources as the experiments in
	Section~\ref{sec:experiments}. Of these runs, 242 terminate with an integer optimal solution,
	17 report infeasibility under the forbidden-word restriction, and 13 reach
	the time limit. A run that reaches the time limit could still prove a word mandatory
	if its integer lower bound exceeds $z^\star$, but this
	did not happen in the present runs. Instead, 5 of the 13 time-limit runs
	proved the forbidden word replaceable, because they found a separator of
	size $z^\star$ without it. The remaining 8 runs are classified as restrictive--inconclusive.
	
	Overall, the 272 tested words are classified as
	\[
	\boxed{113\ \text{core/mandatory},\qquad
		151\ \text{replaceable},\qquad
		8\ \text{restrictive--inconclusive}}.
	\]
	Thus, 41.54\% of the tested words are proved to belong to the optimal core,
	55.51\% are proved replaceable, and only 2.94\% remain unresolved. Among the
	113 mandatory words, 96 are objective-critical and 17 are
	feasibility-critical.
	
	Table~\ref{tab:nlp-node-roles} summarizes the instance-level structure. The
	quantity $
	\rho_{\mathrm{cert}}=\frac{|C_{\mathrm{cert}}|}{z^\star}$ is the fraction of the reference optimum already proved to belong to the core.
	When there are no inconclusive words, this is the exact rigidity ratio; when
	inconclusive words remain, it is a certified lower bound on the true rigidity.
	The classification refers to the candidate set after the preprocessing of
	Section~\ref{sec:nlp-motivation}, from which a word is removed when another
	word separates every pair it separates. Restoring the removed words changes
	the status of one core word. On $(250, 42, 3)$ the word \emph{vanderbilt}
	is present in exactly the same texts as \emph{biltmore} and was removed as
	its duplicate, so the two words are interchangeable. Of the other 112 core
	words, only \emph{lolita} on $(100, 43, 2)$ has a removed word that
	separates no pair which \emph{lolita} does not separate, namely
	\emph{film}, and an exact check confirms that \emph{lolita} remains in
	every optimum when the removed words are restored.
	
	\begin{table}[htbp]
		\centering
		\caption{Structural classification of the 272 words belonging to the 19
			certified \textsc{Bbc} optima. ``Core'' denotes words proved mandatory in every
			optimum; ``Repl.'' denotes replaceable words; ``Inc.'' denotes
			restrictive--inconclusive words.}
		\label{tab:nlp-node-roles}
		\scriptsize
		\begin{tabular}{lrrrrr}
			\toprule
			Instance $(\qns, \mathit{seed}, k)$ & $z^\star$ & Core & Repl. & Inc. & $\rho_{\mathrm{cert}}$ \\
			\midrule
			$(100, 42, 2)$ & 12 & 2 & 10 & 0 & 0.17 \\
			$(100, 42, 3)$ & 15 & 13 & 2 & 0 & 0.87 \\
			$(100, 43, 1)$ & 10 & 1 & 9 & 0 & 0.10 \\
			$(100, 43, 2)$ & 14 & 2 & 12 & 0 & 0.14 \\
			$(100, 43, 3)$ & 15 & 7 & 8 & 0 & 0.47 \\
			$(100, 44, 1)$ & 11 & 0 & 11 & 0 & 0.00 \\
			$(100, 44, 2)$ & 11 & 6 & 5 & 0 & 0.55 \\
			$(100, 44, 3)$ & 14 & 6 & 8 & 0 & 0.43 \\
			$(100, 47, 1)$ & 11 & 0 & 11 & 0 & 0.00 \\
			$(100, 48, 1)$ & 11 & 0 & 11 & 0 & 0.00 \\
			$(100, 49, 1)$ & 11 & 5 & 0 & 6 & 0.45 \\
			\boxed{$(100, 50, 1)$} & 10 & 10 & 0 & 0 & 1.00 \\
			$(100, 51, 1)$ & 11 & 0 & 11 & 0 & 0.00 \\
			$(150, 48, 1)$ & 14 & 0 & 13 & 1 & 0.00 \\
			$(200, 48, 1)$ & 15 & 0 & 14 & 1 & 0.00 \\
			$(250, 42, 2)$ & 18 & 5 & 13 & 0 & 0.28 \\
			$(250, 42, 3)$ & 25 & 15 & 10 & 0 & 0.60 \\
			\boxed{$(250, 43, 2)$} & 18 & 18 & 0 & 0 & 1.00 \\
			$(250, 43, 3)$ & 26 & 23 & 3 & 0 & 0.88 \\
			\midrule
			Total & 272 & 113 & 151 & 8 & -- \\
			\bottomrule
		\end{tabular}
	\end{table}
	
	The experiment reveals several qualitatively different types of optimal
	solution structure. Two instances are completely rigid: $(100, 50, 1)$ and $(250, 43, 2)$ have respectively $10/10$ and $18/18$
	reference-optimum words proved mandatory. Since every word of a size-$z^\star$
	optimum belongs to every optimum, the optimal separator is unique in both
	cases. Strong but not complete rigidity is observed for $(100, 42, 3)$ and $(250, 43, 3)$. 
	%\texttt{hc3\_q100\_seed42\_k3} ($13/15$ mandatory) and
	%\texttt{hc3\_q250\_seed43\_k3} ($23/26$ mandatory).
	
	At the opposite extreme, every word in the reference optima of the instances
	$(100, 44, 1)$, $(100, 47, 1)$, $(100, 48, 1)$, and $(100, 51, 1)$ has been
	proved replaceable. Consequently, their optimal cores are empty: no word from
	the reference optimum occurs in every optimal separator. The instances
	$(150, 48, 1)$ and $(200, 48, 1)$ are also highly flexible, with 13 of 14 and
	14 of 15 words proved replaceable, respectively, while one word in each
	instance remains inconclusive.
	
	The most ambiguous instance is $(100, 49, 1)$. Five words are already proved mandatory
	(\texttt{as}, \texttt{error}, \texttt{on}, \texttt{which}, and
	\texttt{with}), while six words remain restrictive--inconclusive. For each
	of these six words, the restricted run finds only a larger incumbent, but its
	integer lower bound does not exclude an alternative optimum of size
	$z^\star=11$.
	
	The two kinds of instances differ clearly. On the nine standard instances
	only 16 of the 104 tested words are in the core, 15 of them on $(100, 49, 1)$ and $(100, 50, 1)$, and six of the
	nine instances have an empty core. On the ten variants 97 of the 168
	tested words are in the core, and all 17 feasibility-critical words occur
	there. On the standard instances every pair is separated by at least three
	words, so an optimum can usually be rearranged without increasing its
	size. The stricter binarization of the variants leaves far fewer
	candidates and some pairs that only a few words separate, and these words
	become indispensable.

	\paragraph{Coverage of the selected words.}
	The classification of Table~\ref{tab:nlp-node-roles} can be read against a
	simple per-word quantity. A word $t$ separates the pair $(r,b)$ if and only
	if it occurs in exactly one of the two texts, so the number of pairs it
	separates is
	\[
	P(t) \;=\; d_R(t)\,\bigl(|B|-d_B(t)\bigr) + \bigl(|R|-d_R(t)\bigr)\,d_B(t),
	\]
	where $d_R(t)$ and $d_B(t)$ are the numbers of human and ChatGPT texts
	containing $t$. We call $P(t)$ the coverage of $t$. Its theoretical maximum
	$|R|\,|B|$ would be attained by a word occurring in every text of one
	class and in none of the other; no such word exists on these instances,
	and the candidates with the highest coverage are words that occur in about
	half of the texts of both classes, each separating close to half of all
	pairs. Figure~\ref{fig:coverage}
	plots the coverage of every candidate word on four instances chosen to
	represent the different solution structures of Table~\ref{tab:nlp-node-roles},
	with the words of the certified optimum $S^\star$ marked by their class.
	
	\begin{figure}[htbp]
		\centering
		\includegraphics[width=\textwidth]{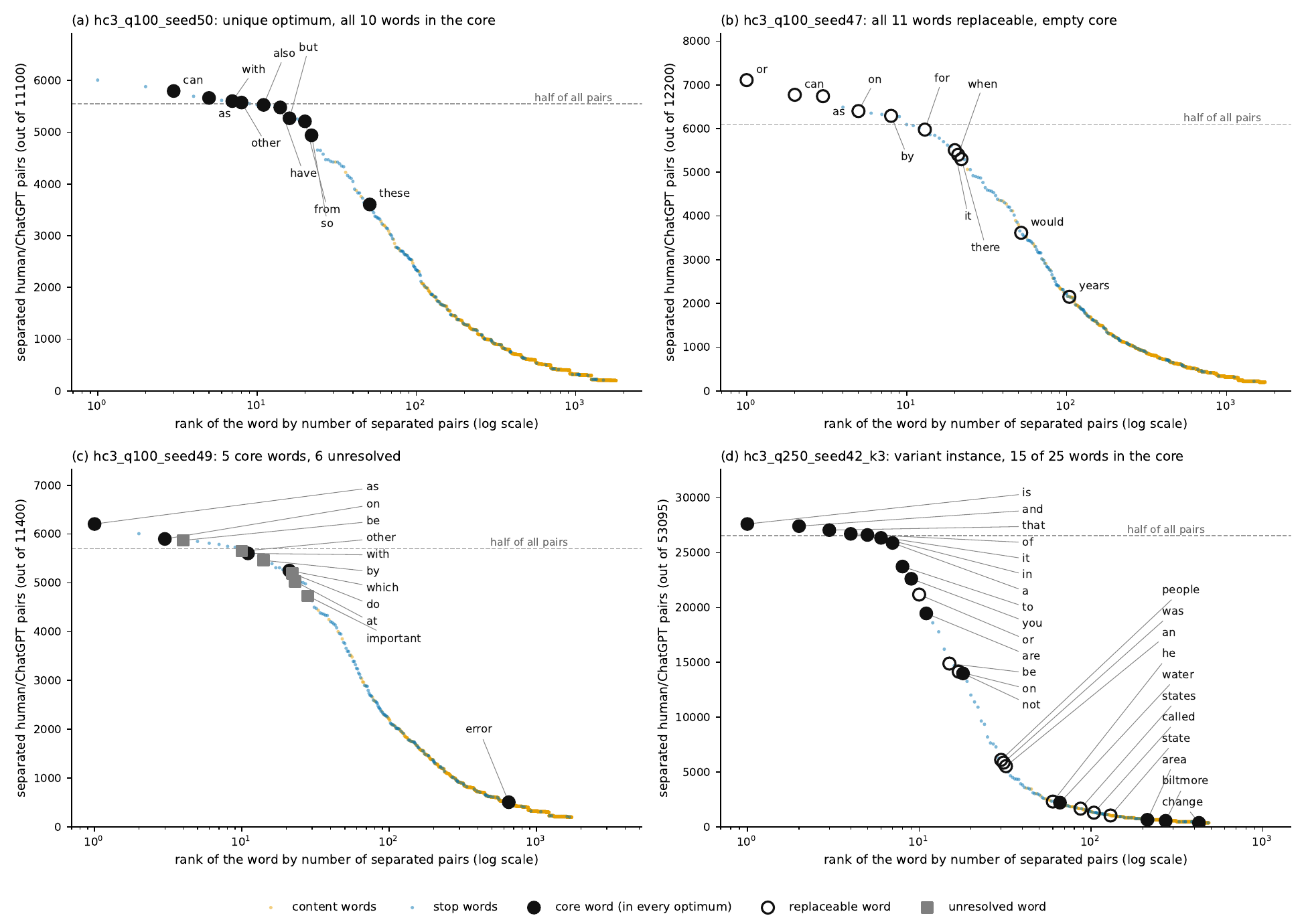}
		\caption{Coverage $P(t)$ of every candidate word, ordered by rank, on four
			NLP instances: (a) the instance $(100, 50, 1)$, whose optimum is unique;
			(b) the instance $(100, 47, 1)$, whose optimum has an empty core;
			(c) the instance $(100, 49, 1)$, with five core words and six unresolved
			words; (d) the instance $(250, 42, 3)$, with 15 of
			25 words in the core. Stop words and content words are drawn in different
			colors; the words of the certified optimum are marked as core (filled),
			replaceable (hollow) or restrictive--inconclusive (grey square). The {dashed horizontal line distinguishes between the words that separate more than half of the human/ChatGPT pairs (above the line) and those that separate fewer.}}
		\label{fig:coverage}
	\end{figure}
	
	Three observations hold on all 19 instances, not only on the four shown. The first is that the optimal separators are built from the
	top of the coverage curve.
	Of the 272 words in the 19 certified optima, 212 are among the 30
	highest-ranked candidates of their instance and 245 are within the top
	10\,\% of the candidates; the median word of an optimum separates 44\,\% of
	all pairs on its own. On the standard instances the top of the curve is
	made of short function words (\emph{or}, \emph{as}, \emph{can}, \emph{on},
	\emph{with}), which occur in roughly half of the texts of both classes, and
	these words make up most of the optimal separators.
	
	The second observation is that the class of a word is not visible from
	its coverage. Core words have a median rank of 11 and
	replaceable words a median rank of 15; both classes separate a median of
	44\,\% of the pairs, and the two profiles overlap completely. The unique
	optimum of $(100, 50, 1)$ in panel (a) and the optimum of $(100, 47, 1)$
	in panel (b), all of whose words are replaceable, occupy the same region
	of the curve. Whether a
	word can be exchanged is therefore decided by \emph{which} pairs it
	separates relative to the other words, not by how many, and this is
	precisely the information that the forbid-one-word certificates add to the
	picture. The third observation concerns the exceptions to the first one, which
	are the words whose contribution no other word provides. The core word \emph{error} on $(100, 49, 1)$ ranks 648th
	of 1\,732 candidates and separates 4\,\% of the pairs; on the instance $(250, 42, 3)$ the core words \emph{area}, \emph{biltmore} (or its duplicate
	\emph{vanderbilt}), \emph{change} and \emph{water} rank between 66th and 425th and separate
	about 1\,\% of the pairs each. These words are mandatory not because they
	separate many pairs but because they separate a few critical pairs that no
	high-coverage word separates; the feasibility-critical words of the
	variants are the extreme case, each being the only separator of some pair.
	The coverage curve thus accounts for the bulk of the optima, while the
	structural classification identifies the few words whose contribution is
	irreplaceable.
	
	% ---------------------------------------------------------------- (B)
	\paragraph{Why the selected words are short.}
	%% ---- DOPUNA ----
	{Function words dominate between the components of optima because the objective minimizes the number of selected words, and on these texts a function word separates far more pairs than a content word.}
	%% ---- kraj dopune ----
	A content
	word that occurs in $d$ texts separates at most $d\max(|R|,|B|)$ pairs,
	a small fraction of the $|R|\,|B|$ pairs when $d$ is small, whereas a
	function word that occurs in about half of the texts of each class
	separates close to half of all pairs; on the instance $(100, 43, 1)$, the median word with more than
	five letters separates fewer than 3\,\% of the pairs. Restricting the
	candidate set makes this explicit. Table~\ref{tab:candidate-restriction}
	reports the size of the separator found by the greedy heuristic described below on the 32
	standard instances when the candidates are restricted to (i) the 143 to
	152 stop words of the instance, (ii) words of at least four, five or six
	letters and (iii) all words except stop words. With stop words alone the
	separator is as small as on the full vocabulary on the 28 instances where
	the stop words admit a separator (median ratio 1.00), so about 6\,\% of
	the vocabulary usually carries the whole separation; on the other four
	instances the stop words do not separate every pair, because some
	human/ChatGPT pair differs only in content words. Excluding stop words
	makes the separator 1.4 to 3.2 times larger (median 1.8), and restricting
	to words of at least six letters leaves no separating set at all on seven
	of the 32 instances. The ratios in the table are taken over the instances
	on which the restricted candidates admit a separator. The one ratio below
	one, 0.93 on the instance $(150, 47, 1)$, where the stop words alone give
	13 words against 14 on the full vocabulary, arises because the greedy
	heuristic is not monotone in the candidate set: a smaller candidate set
	can lead it along a different path to a smaller separator. 
	
	Exact solutions of the restricted problems on the
	standard instances with a certified optimum confirm that this is not a
	weakness of the heuristic: with words of at least five letters the
	optimum is 15 to 18 on the $\qns=100$ instances, 30 on $(150, 48, 1)$ and
	36 on $(200, 48, 1)$, against 10 to 15 on the full vocabulary, while
	stop words alone attain the full-vocabulary optimum on four of these
	instances and are within one word of it on three others. Making content
	words competitive with high-coverage function words would therefore
	require a change of the objective, for instance weighting words by their
	length or their document frequency; we leave this to future work and keep the pair-separation objective, whose solutions are
	what the exact methods certify.
	
	The heuristic used for Table~\ref{tab:candidate-restriction} is the
	classical greedy rule for set covering~\cite{Chvatal1979}, applied to the
	family of sets $D(r,b)$ restricted to the admissible candidates. Starting
	from the empty set, it repeatedly adds the admissible word that separates
	the largest number of pairs not yet separated. These pairs are exactly the
	red--blue pairs whose documents have the same incidence pattern on the
	words chosen so far, so the gain of a word is its coverage $P(t)$ computed
	within each group of such documents and summed over the groups. Ties are
	broken by a fixed order of the words, so the result is deterministic. The
	procedure stops when every pair is separated. If unseparated pairs remain
	but no admissible word separates any of them, the restricted instance has
	no feasible separator, and this is how the column \emph{infeasible} of
	Table~\ref{tab:candidate-restriction} is obtained. The size of the greedy
	separator exceeds the optimum by at most the factor
	$\mathcal{H}(\max_t P(t))\le 1+\ln(|R|\,|B|)$, where $\mathcal{H}(n)$
	denotes the $n$-th harmonic number~\cite{Chvatal1979}. It differs from the
	randomized construction used by GRASP
	in Section~\ref{sec:grasp-construct}, and the same rule is applied to the
	full and to every restricted candidate set, so the ratios in
	Table~\ref{tab:candidate-restriction} compare solutions of the same
	heuristic.
	
	\begin{table}[htbp]
		\centering
		\caption{Size of the greedy separator when the candidate set is restricted,
			relative to the greedy separator on the full vocabulary, over the 32
			standard instances. \emph{share}: mean share of the vocabulary that
			remains a candidate; \emph{infeasible}: instances for which no feasible
			separator exists within the restricted candidate set. The ratios are
			taken over the remaining instances.}
		\label{tab:candidate-restriction}
		\small
		\begin{tabular}{lrrr}
			\toprule
			candidate words & share & ratio min / median / max & infeasible \\
			\midrule
			all words                   & 1.00 & 1.00 / 1.00 / 1.00 & 0 \\
			stop words only             & 0.06 & 0.93 / 1.00 / 1.23 & 4 \\
			at least 4 letters          & 0.92 & 1.08 / 1.25 / 1.69 & 0 \\
			at least 5 letters          & 0.78 & 1.33 / 1.64 / 2.75 & 0 \\
			at least 6 letters          & 0.62 & 1.75 / 2.29 / 3.50 & 7 \\
			all words except stop words & 0.94 & 1.42 / 1.84 / 3.19 & 0 \\
			\bottomrule
		\end{tabular}
	\end{table}

	These results show that membership in one minimum separator is not, by itself,
	a sufficient measure of feature importance. A \emph{replaceable} word can be
	removed while preserving the optimum cardinality, whereas a
	\emph{core/mandatory} word is structurally unavoidable across the complete
	family of optimal separators. The latter therefore provides a stronger form
	of vocabulary importance. Conversely, the restrictive--inconclusive class
	identifies words for which the current computation provides evidence of
	importance but not a proof of core membership. This distinction is useful for
	interpreting the NLP model because it separates genuinely unavoidable
	discriminative vocabulary from alternative, mutually substitutable optimal
	explanations.
	
	%\subsection{NLP benchmark conclusions}
	%\label{sec:nlp-conclusions}
	
	%The complete 42-instance tuning study selects
	%$q_{\mathrm{BBC}}=2$ and $q_{\mathrm{LBBD}}=100$. The newly added
	%$q=500$ experiments confirm that neither algorithm benefits from the most
	%aggressive cut batch: for BBC the mean gap rises to
	%16.26\%, while for %LBBD it rises to 15.23\%.
	%On the final 42-instance benchmark, ILP remains the strongest primal exact
	%method, BBC achieves the smallest mean gap
	%(13.21\%) and the largest number of
	%certificates (16), while LBBD provides
	%the strongest mean lower bound (12.29).
	%GRASP remains the stronger repeated-run heuristic baseline overall, with mean
	%solution size 13.73 versus
	%14.30 for VNS. The results therefore reinforce
	%the complementary roles of the approaches: ILP is strongest on primal quality,
	%BBC on final gap and certification, and LBBD on dual-bound strength.
	
	\section{Conclusions}\label{sec:conclusions}
	
	This paper develops and experimentally evaluates a systematic exact and hybrid computational
	framework for the Minimum Red--Blue Separating Set Problem (MRBSSP). The
	framework combines a compact Set-Cover ILP formulation, safe preprocessing with kernelization,
	an iterative Logic-Based Benders Decomposition (LBBD), and a
	Branch-and-Benders-Cut (BBC) algorithm. GRASP approach is used both as a standalone
	heuristic and as a source of validated warm starts and upper bounds for the
	exact approaches, while the VNS approach provides an additional literature-based heuristic
	baseline. The computational study is deliberately organized around two
	complementary benchmark families: a broad \textsc{Random} graph benchmark for
	controlled algorithmic stress testing and an NLP-derived benchmark that
	connects MRBSSP with interpretable vocabulary selection for distinguishing
	human- and AI-generated texts.
	
	The \textsc{Random} benchmark contains 250 instances arranged across graph sizes and
	densities, of which 200 are feasible. This benchmark exposes clear
	complementarity among the exact approaches. The best parameter choice of the BBC approach adds at most $q_{\mathrm{BBC}}=5$ lazy cuts per callback and LBBD at most $q_{\mathrm{LBBD}}=100$ cuts per master iteration. ILP gives the
	strongest primal performance, with mean a UB of 18.83. BBC proves the largest
	number of instances optimal, 92 out of 200, whereas LBBD gives the strongest
	mean lower bound, 13.63, and the smallest mean optimality gap, 24.37\%.
	On the 11 groups not solved completely by all exact methods, the mean group
	gaps are 53.71\% for ILP, 47.21\% for BBC, and 44.30\% for LBBD. The
	repeated-run heuristic study additionally shows that both meta-heuristic approaches, GRASP and VNS, 
	produce strong feasible solutions with low run-to-run variability. Thus, the
	\textsc{Random} benchmark provides a broad view of the trade-off between primal
	quality, dual-bound progress, and certification strength.
	
	The NLP benchmark provides a second and substantially different test of the
	same algorithms. The complete parameter study over
	$q_{\text{LBBD}}, q_{\text{BBC}} \in\{1,2,5,10,50,100,500\}$ selects
	$q_{\mathrm{BBC}}=2$ and $q_{\mathrm{LBBD}}=100$. Across all 42 NLP
	instances, ILP remains the strongest primal exact method, with a mean upper
	bound of 13.62. BBC attains the smallest mean optimality gap (13.21\%) and
	proves the largest number of optima (16), while LBBD provides the
	strongest mean lower bound of 12.29. The
	$q_{\text{LBBD}}=500$ experiments confirm that increasing cut batches beyond the selected
	values does not improve the final trade-off.
	
	The NLP benchmark also demonstrates that the value of an exact solution is
	not limited to its objective cardinality. By forbidding vertices from
	certified optimal separators one at a time, we distinguish mandatory
	vocabulary features from replaceable ones. The structural experiment covers
	the 19 NLP instances for which BBC certifies optimality under the extended
	computational limit, resulting in 272 restricted BBC runs. Overall, 113 of
	the tested vocabulary features are proved to be core/mandatory, 151 are
	proved replaceable, and only 8 remain restrictive--inconclusive. Thus,
	approximately 41.54\% of the tested words are certified as belonging to the
	optimal core, whereas 55.51\% are replaceable and 2.94\% remain unresolved.
	Among the 113 mandatory words, 96 are objective-critical and 17 are
	feasibility-critical. Some instances admit highly non-unique optimal
	vocabularies, including cases with an empty certified core, whereas others
	exhibit completely or almost completely rigid optimal separators. This
	structural analysis provides a stronger interpretation of feature importance
	than reporting a single minimum separator, since it distinguishes features
	that are unavoidable across all optimal solutions from those that can be
	substituted without increasing the objective value.
	
	Taken together, the two benchmark families provide complementary evidence.
	The study on the \textsc{Random} benchmark emphasizes robustness across graph size and density, whereas
	the NLP study demonstrates that the same exact and hybrid methodology remains
	effective on structured real-data graphs and can additionally support an
	interpretable analysis of the complete optimal-solution structure. Across
	both benchmarks, no single exact method dominates every computational
	criterion: ILP is generally strongest on primal solution quality, LBBD is
	particularly effective in producing stronger dual bounds, and BBC is highly
	competitive in optimality certification. The heuristic approaches complement
	the exact methods by supplying high-quality feasible solutions and, in the
	case of GRASP, useful warm starts and upper bounds for the exact search.
	
	Future work should enlarge both benchmark families, especially by adding more
	structurally diverse real-data instances, and should investigate stronger
	problem-specific reductions, adaptive cut-batch policies for BBC and LBBD,
	and alternative cut-ranking strategies. On the application side, natural
	extensions include weighted or cost-sensitive vocabulary separators,
	robustness to perturbations of the text-feature incidence matrix, and
	multi-class generalizations in which several text-generation sources must be
	distinguished simultaneously. A further direction is to exploit the
	mandatory/replaceable feature analysis directly within the optimization
	process, for example to identify persistent core features across multiple
	optimal or near-optimal separators.
	
	\section*{Acknowledgements}
	This publication is co-funded by the European Union's Horizon Europe research and innovation programme under the Marie Sk\l{}odowska-Curie COFUND Postdoctoral Programme, grant agreement No.~101081355--SMASH, and by the Republic of Slovenia and the European Union from the European Regional Development Fund. Co-funded by the European Union. Views and opinions expressed are those of the authors only and do not necessarily reflect those of the European Union or the European Research Executive Agency (REA). Neither the European Union nor the REA can be held responsible for them.

	\appendix
	
	\section{Numerical results of $q_{\texttt{BBC}}$ and $q_{\texttt{LBBD}}$ parameter tuning on the benchmark set \textsc{Random}}\label{sec:appendix-random-tuning}
	
	\begin{table}[!htb]
		\centering
		\caption{\textsc{Bbc} approach tuned on 20  selected \textsc{Random} instances.}
		\label{tab:tuning-random-bbc}
		\small
		\begin{tabular}{rrrrr}
			\toprule
			$q_{\mathrm{BBC}}$ & Mean UB & Mean LB & Mean gap (\%) & Opt.\\
			\midrule
			1 & 19.15 & 13.00 & 26.17 & 9 \\
			2 & 19.10 & 13.05 & 26.26 & 9 \\
			\fbox{{5}} & 19.20 & 13.15 & \textbf{25.39} & \textbf{10} \\
			10 & 19.05 & 13.05 & 26.41 & 9 \\
			50 & 19.00 & 13.00 & 26.35 & 9 \\
			100 & 19.20 & 13.00 & 27.10 & 9 \\
			500 & 19.25 & 13.05 & 26.16 & 9 \\
			\bottomrule
		\end{tabular}
	\end{table}
	
	\begin{table}[!htb]
		\centering
		\caption{\textsc{Lbbd} approach tuned on  20 selected \textsc{Random} instances.}
		\label{tab:tuning-random-lbbd}
		\small
		\begin{tabular}{rrrrr}
			\toprule
			$q_{\mathrm{LBBD}}$ & Mean UB & Mean LB & Mean gap (\%) & Opt.\\
			\midrule
			1 & 20.40 & 13.15 & 27.37 & 8 \\
			2 & 20.45 & 13.30 & 26.50 & 8 \\
			5 & 20.10 & 13.45 & 25.61 & 8 \\
			10 & 20.45 & 13.50 & 26.30 & 8 \\
			50 & 20.35 & 13.65 & 25.51 & 8 \\
			\fbox{100} & 20.35 & 13.55 & \textbf{24.85} & \textbf{9} \\
			500 & 20.15 & 13.45 & 25.86 & 8 \\
			\bottomrule
		\end{tabular}
	\end{table}
	
	\section{Numerical results of $q_{\texttt{BBC}}$ and $q_{\texttt{LBBD}}$ parameter tuning on the NLP problem instances}\label{sec:appendix-nlp-tuning} \vspace{1cm}
	
	\begin{table}[!htb]
		\centering
		\caption{\textsc{Bbc} tuned on all 42 NLP benchmark instances. ``Rounds'' denotes the
			mean number of callback invocations. The selected setting is shown in bold.}
		\label{tab:tuning-nlp-bbc}
		\small
		\begin{tabular}{rrrrrr}
			\toprule
			$q_{\mathrm{BBC}}$ & Mean UB & Mean LB & Mean gap (\%) & Opt. & Rounds\\
			\midrule
			1 & 13.95 & 12.07 & 13.78 & 15 & 1640.0 \\
			\fbox{2} & \textbf{13.90} & \textbf{12.12} & \textbf{13.21} & \textbf{16} & \textbf{841.5} \\
			5 & 13.86 & 11.98 & 13.99 & 15 & 346.6 \\
			10 & 13.86 & 12.05 & 13.29 & 16 & 178.0 \\
			50 & 14.00 & 11.93 & 15.37 & 12 & 75.6 \\
			100 & 13.98 & 11.81 & 16.25 & 11 & 62.2 \\
			500 & 13.98 & 11.81 & 16.26 & 12 & 50.5 \\
			\bottomrule
		\end{tabular}
	\end{table}

	\begin{table}[!htb]
		\centering
		\caption{\textsc{Lbbd} tuned on all 42 NLP benchmark instances. ``Rounds'' denotes the
			mean number of master iterations. The selected setting is shown in bold.}
		\label{tab:tuning-nlp-lbbd}
		\small
		\begin{tabular}{rrrrrr}
			\toprule
			$q_{\mathrm{LBBD}}$ & Mean UB & Mean LB & Mean gap (\%) & Opt. & Rounds\\
			\midrule
			1 & 14.17 & 11.45 & 20.12 & 10 & 521.2 \\
			2 & 14.29 & 11.74 & 18.61 & 10 & 304.4 \\
			5 & 14.29 & 11.95 & 17.21 & 10 & 145.4 \\
			10 & 14.17 & 12.07 & 15.53 & 10 & 83.0 \\
			50 & 14.21 & {12.29} & 14.20 & 10 & 25.4 \\
			\fbox{100} & \textbf{14.26} & \textbf{12.29} & \textbf{14.38} & \textbf{11} & \textbf{18.4} \\
			500 & 14.29 & 12.19 & 15.23 & 10 & 11.8 \\
			\bottomrule
		\end{tabular}
	\end{table}

	\section{Word-level classification of words in the optimal solutions of NLP instances solved by \textsc{Bbc}}
	
	The following lists give the status of every tested word. Node IDs are shown
	in parentheses. A dagger ($^\dagger$) marks a feasibility-critical core word;
	all other core words are objective-critical.
	
	\paragraph{$(100, 42, 2)$ ($z^\star=12$).}
	\textbf{Core/mandatory:} \texttt{more} (471), \texttt{on} (495).\\
	\textbf{Replaceable:} \texttt{are} (240), \texttt{as} (243), \texttt{be} (250), \texttt{help} (390), \texttt{if} (402), \texttt{in} (406), \texttt{is} (427), \texttt{of} (491), \texttt{will} (678), \texttt{your} (685).\\
	\textbf{Restrictive--inconclusive:} --.
	
	\paragraph{$(100, 42, 3)$ ($z^\star=15$).}
	\textbf{Core/mandatory:} \texttt{a} (191), \texttt{are} (207), \texttt{as} (209), \texttt{be} (211), \texttt{in} (276), \texttt{is} (288), \texttt{it} (289), \texttt{of} (316), \texttt{on} (318), \texttt{or} (321), \texttt{that} (375), \texttt{to} (386), \texttt{you} (415).\\
	\textbf{Replaceable:} \texttt{create} (230), \texttt{when} (406).\\
	\textbf{Restrictive--inconclusive:} --.
	
	\paragraph{$(100, 43, 1)$ ($z^\star=10$).}
	\textbf{Core/mandatory:} \texttt{or} (1380).\\
	\textbf{Replaceable:} \texttt{also} (315), \texttt{as} (355), \texttt{between} (423), \texttt{but} (463), \texttt{from} (927), \texttt{important} (1056), \texttt{like} (1178), \texttt{when} (1999), \texttt{you} (2037).\\
	\textbf{Restrictive--inconclusive:} --.
	
	\paragraph{$(100, 43, 2)$ ($z^\star=14$).}
	\textbf{Core/mandatory:} \texttt{i} (460)$^\dagger$, \texttt{lolita} (510)$^\dagger$.\\
	\textbf{Replaceable:} \texttt{and} (235), \texttt{as} (244), \texttt{be} (259), \texttt{by} (277), \texttt{can} (279), \texttt{for} (407), \texttt{foreign} (409), \texttt{it} (485), \texttt{or} (567), \texttt{that} (694), \texttt{which} (752), \texttt{you} (772).\\
	\textbf{Restrictive--inconclusive:} --.
	
	\paragraph{$(100, 43, 3)$ ($z^\star=15$).}
	\textbf{Core/mandatory:} \texttt{and} (203), \texttt{for} (274), \texttt{from} (277), \texttt{of} (351), \texttt{that} (412), \texttt{to} (426), \texttt{you} (455).\\
	\textbf{Replaceable:} \texttt{a} (194), \texttt{guidance} (292), \texttt{in} (306), \texttt{more} (340), \texttt{not} (347), \texttt{on} (352), \texttt{the} (413), \texttt{was} (439).\\
	\textbf{Restrictive--inconclusive:} --.
	
	\paragraph{$(100, 44, 1)$ ($z^\star=11$).}
	\textbf{Core/mandatory:} --.\\
	\textbf{Replaceable:} \texttt{for} (833), \texttt{have} (915), \texttt{how} (959), \texttt{important} (978), \texttt{in} (981), \texttt{more} (1190), \texttt{on} (1264), \texttt{there} (1734), \texttt{they} (1737), \texttt{was} (1840), \texttt{with} (1873).\\
	\textbf{Restrictive--inconclusive:} --.
	
	\paragraph{$(100, 44, 2)$ ($z^\star=11$).}
	\textbf{Core/mandatory:} \texttt{as} (233), \texttt{not} (499), \texttt{on} (508), \texttt{that} (616), \texttt{with} (682), \texttt{you} (689).\\
	\textbf{Replaceable:} \texttt{are} (231), \texttt{be} (246), \texttt{is} (435), \texttt{their} (618), \texttt{to} (632).\\
	\textbf{Restrictive--inconclusive:} --.
	
	\paragraph{$(100, 44, 3)$ ($z^\star=14$).}
	\textbf{Core/mandatory:} \texttt{about} (185)$^\dagger$, \texttt{and} (192), \texttt{in} (268), \texttt{is} (275), \texttt{it} (276), \texttt{to} (371).\\
	\textbf{Replaceable:} \texttt{a} (184), \texttt{for} (246), \texttt{if} (266), \texttt{of} (308), \texttt{on} (311), \texttt{that} (360), \texttt{their} (362), \texttt{you} (402).\\
	\textbf{Restrictive--inconclusive:} --.
	
	\paragraph{$(100, 47, 1)$ ($z^\star=11$).}
	\textbf{Core/mandatory:} --.\\
	\textbf{Replaceable:} \texttt{as} (358), \texttt{by} (465), \texttt{can} (476), \texttt{for} (890), \texttt{it} (1091), \texttt{on} (1318), \texttt{or} (1332), \texttt{there} (1786), \texttt{when} (1915), \texttt{would} (1944), \texttt{years} (1950).\\
	\textbf{Restrictive--inconclusive:} --.
	
	\paragraph{$(100, 48, 1)$ ($z^\star=11$).}
	\textbf{Core/mandatory:} --.\\
	\textbf{Replaceable:} \texttt{about} (228), \texttt{if} (978), \texttt{make} (1125), \texttt{not} (1228), \texttt{on} (1253), \texttt{or} (1272), \texttt{other} (1278), \texttt{so} (1581), \texttt{this} (1693), \texttt{use} (1768), \texttt{with} (1834).\\
	\textbf{Restrictive--inconclusive:} --.
	
	\paragraph{$(100, 49, 1)$ ($z^\star=11$).}
	\textbf{Core/mandatory:} \texttt{as} (344), \texttt{error} (745), \texttt{on} (1278), \texttt{which} (1907), \texttt{with} (1918).\\
	\textbf{Replaceable:} --.\\
	\textbf{Restrictive--inconclusive:} \texttt{at} (351), \texttt{be} (377), \texttt{by} (441), \texttt{do} (678), \texttt{important} (983), \texttt{other} (1301).
	
	\paragraph{$(100, 50, 1)$ ($z^\star=10$).}
	\textbf{Core/mandatory:} \texttt{also} (292), \texttt{as} (340), \texttt{but} (452), \texttt{can} (465), \texttt{from} (899), \texttt{have} (967), \texttt{other} (1350), \texttt{so} (1697), \texttt{these} (1823), \texttt{with} (1972).\\
	\textbf{Replaceable:} --.\\
	\textbf{Restrictive--inconclusive:} --.
	
	\paragraph{$(100, 51, 1)$ ($z^\star=11$).}
	\textbf{Core/mandatory:} --.\\
	\textbf{Replaceable:} \texttt{and} (330), \texttt{as} (357), \texttt{at} (365), \texttt{but} (455), \texttt{can} (467), \texttt{for} (855), \texttt{no} (1213), \texttt{on} (1240), \texttt{other} (1257), \texttt{this} (1694), \texttt{when} (1807).\\
	\textbf{Restrictive--inconclusive:} --.
	
	\paragraph{$(150, 48, 1)$ ($z^\star=14$).}
	\textbf{Core/mandatory:} --.\\
	\textbf{Replaceable:} \texttt{an} (448), \texttt{are} (486), \texttt{for} (1218), \texttt{has} (1336), \texttt{i} (1399), \texttt{in} (1414), \texttt{it} (1491), \texttt{like} (1563), \texttt{of} (1779), \texttt{on} (1790), \texttt{their} (2439), \texttt{which} (2628), \texttt{with} (2646).\\
	\textbf{Restrictive--inconclusive:} \texttt{you} (2671).
	
	\paragraph{$(200, 48, 1)$ ($z^\star=15$).}
	\textbf{Core/mandatory:} --.\\
	\textbf{Replaceable:} \texttt{also} (573), \texttt{an} (591), \texttt{as} (655), \texttt{for} (1640), \texttt{important} (1887), \texttt{not} (2334), \texttt{of} (2368), \texttt{on} (2384), \texttt{that} (3238), \texttt{there} (3247), \texttt{they} (3250), \texttt{with} (3498), \texttt{you} (3527), \texttt{your} (3530).\\
	\textbf{Restrictive--inconclusive:} \texttt{in} (1891).
	
	\paragraph{$(250, 42, 2)$ ($z^\star=18$).}
	\textbf{Core/mandatory:} \texttt{by} (656), \texttt{i} (979)$^\dagger$, \texttt{in} (991), \texttt{one} (1189)$^\dagger$, \texttt{you} (1600).\\
	\textbf{Replaceable:} \texttt{a} (532), \texttt{and} (574), \texttt{as} (588), \texttt{for} (881), \texttt{have} (945), \texttt{is} (1028), \texttt{it} (1030), \texttt{not} (1170), \texttt{on} (1187), \texttt{or} (1194), \texttt{that} (1477), \texttt{there} (1484), \texttt{to} (1500).\\
	\textbf{Restrictive--inconclusive:} --.
	
	\paragraph{$(250, 42, 3)$ ($z^\star=25$).}
	\textbf{Core/mandatory:} \texttt{a} (476), \texttt{and} (490), \texttt{are} (495), \texttt{area} (496)$^\dagger$, \texttt{biltmore} (520)$^\dagger$, \texttt{change} (543)$^\dagger$, \texttt{in} (683), \texttt{is} (704)$^\dagger$, \texttt{it} (705), \texttt{not} (762), \texttt{of} (767)$^\dagger$, \texttt{that} (896), \texttt{to} (912)$^\dagger$, \texttt{water} (936), \texttt{you} (956).\\
	\textbf{Replaceable:} \texttt{an} (489), \texttt{be} (509), \texttt{called} (533), \texttt{he} (660), \texttt{on} (769), \texttt{or} (773), \texttt{people} (790), \texttt{state} (876), \texttt{states} (877), \texttt{was} (934).\\
	\textbf{Restrictive--inconclusive:} --.
	
	\paragraph{$(250, 43, 2)$ ($z^\star=18$).}
	\textbf{Core/mandatory:} \texttt{can} (666), \texttt{for} (933), \texttt{have} (1006), \texttt{i} (1045)$^\dagger$, \texttt{in} (1058), \texttt{is} (1094), \texttt{it} (1098), \texttt{lolita} (1149)$^\dagger$, \texttt{more} (1206), \texttt{not} (1238), \texttt{on} (1256), \texttt{or} (1264), \texttt{that} (1530), \texttt{they} (1539), \texttt{to} (1551), \texttt{was} (1610), \texttt{which} (1627), \texttt{you} (1651).\\
	\textbf{Replaceable:} --.\\
	\textbf{Restrictive--inconclusive:} --.
	
	\paragraph{$(250, 43, 3)$ ($z^\star=26$).}
	\textbf{Core/mandatory:} \texttt{a} (488), \texttt{an} (502), \texttt{and} (503)$^\dagger$, \texttt{as} (514), \texttt{film} (649), \texttt{for} (660), \texttt{has} (692), \texttt{his} (706), \texttt{i} (713)$^\dagger$, \texttt{in} (720), \texttt{is} (734), \texttt{it} (735), \texttt{not} (810), \texttt{of} (819), \texttt{penguins} (845)$^\dagger$, \texttt{star} (930), \texttt{that} (954), \texttt{the} (955)$^\dagger$, \texttt{they} (962), \texttt{to} (969), \texttt{water} (994), \texttt{you} (1014), \texttt{your} (1015).\\
	\textbf{Replaceable:} \texttt{banks} (524), \texttt{guidance} (688), \texttt{was} (993).\\
	\textbf{Restrictive--inconclusive:} --.

\end{document}